\documentclass[conference]{IEEEtran}

\usepackage{times}

\usepackage[numbers,sort]{natbib}

\usepackage[utf8]{inputenc}             %
\usepackage[T1]{fontenc}                %
\usepackage[bookmarks=true]{hyperref}
\usepackage{microtype}                  %
\usepackage{caption}
\usepackage{cancel}
\usepackage{bigints}
\usepackage{subcaption}
\usepackage{booktabs}

\usepackage{siunitx}

\usepackage{amsthm}

\usepackage{multicol}

\let\labelindent\relax
\usepackage{enumitem}
\setlist{leftmargin=4.0mm}

\usepackage{graphics} %
\usepackage{graphicx} %
\usepackage{amsmath} %
\usepackage{amssymb}  %
\usepackage{amsmath, amssymb,bbm,bm,mathtools}
\usepackage[ruled,vlined]{algorithm2e}
\usepackage{tikz}
\usetikzlibrary{arrows,shapes,automata,backgrounds,petri}
\usepackage{textcomp}
\usepackage{xcolor}
\usepackage{xcolor-material}
\usepackage{soul}
\usepackage{url}
\usepackage{subfiles}
\usepackage{cleveref}
\usepackage[]{mdframed}
\usepackage{float}

\usepackage{dsfont}
\usepackage[absolute]{textpos}
\usepackage{roboto}

\Crefname{figure}{Fig.}{Figs.}

\mdfdefinestyle{ThmFrame}{%
    linecolor=white,%
    outerlinewidth=0pt,%
    roundcorner=2pt,%
    leftmargin=-3pt,%
    rightmargin=-3pt,%
    innertopmargin=2pt, %
    innerbottommargin=6pt, %
    innerrightmargin=8pt, %
    innerleftmargin=8pt, %
    skipabove=2pt,
    skipbelow=-4pt,
    backgroundcolor=MaterialBlueGrey100!15}

\DeclareMathOperator{\ESS}{ESS}
\DeclareMathOperator{\BoldESS}{\mathbf{ESS}}

\newcommand*\diff{\mathop{}\!\mathrm{d}}

\newcommand*\ind[1]{\mathds{1}_{#1}}

\newcommand{\vc}{{\bf c}}

\newcommand{\vu}{{\bf u}}
\newcommand{\vv}{{\bf  v}}

\newcommand{\vw}{{\bf w}}

\newcommand{\tire}{\mathrm{tire}}
\newcommand{\rB}{\mathrm{B}}
\newcommand{\rC}{\mathrm{C}}
\newcommand{\rD}{\mathrm{D}}

\newcommand{\T}{^\mathsf{T}}
\newcommand{\E}{\mathbb{E}}

\newcommand{\norm}[1]{\left\lVert#1\right\rVert}

\newcommand{\KL}[2]{\mathbb{D}_{\mathrm{KL}}\bigl({#1}\parallel {#2}\bigr)}

\usepackage{xcolor,colortbl}
\definecolor{Gray}{gray}{0.65}
\newcolumntype{M}{>{\centering\arraybackslash}m{\dimexpr.225\linewidth-2\tabcolsep}}
\newcolumntype{G}{>{\columncolor{Gray}}M}
\newcolumntype{N}{>{\centering\arraybackslash}m{\dimexpr.1\linewidth-2\tabcolsep}}

\DeclareMathOperator*{\diag}{diag} %

\DeclareMathOperator{\Var}{Var}

\newcommand*\circled[1]{\tikz[baseline=(char.base)]{
            \node[shape=circle,draw,inner sep=2pt] (char) {#1};}}
\usepackage{pifont}
\usepackage{xcolor}   %
\definecolor{C0}{HTML}{E24A33}
\definecolor{C1}{HTML}{348ABD}
\newcommand{\cmark}{\textcolor{C1}{\ding{51}}}%
\newcommand{\xmark}{\textcolor{C0}{\ding{55}}}%
\DeclarePairedDelimiter\abs{\lvert}{\rvert}%

\definecolor{tabOrange}{HTML}{FF7F0E}
\definecolor{tabGreen}{HTML}{2CA02C}

\newcommand{\beforetextbf}{\vspace{0.08\baselineskip}}
\newcommand{\beforetextbfok}{\vspace{0.12\baselineskip}}
\newcommand{\beforetextbfmore}{\vspace{0.20\baselineskip}}

\theoremstyle{plain}

\newtheorem{theorem}{Theorem}
\newtheorem{corollary}{Corollary}
\newtheorem{lemma}[theorem]{Lemma}

\makeatletter
\def\th@newremark{\th@remark\thm@headfont{\bfseries}}
\makeatletter

\theoremstyle{newremark}
\newtheorem{remark}{Remark}

\IEEEoverridecommandlockouts                              %

\title{Safe Beyond the Horizon: Efficient Sampling-based MPC with Neural Control Barrier Functions}

\author{Ji Yin$^{1*}$, Oswin So$^{2*}$, Eric Yang Yu$^{2}$, Chuchu Fan$^{2}$, and Panagiotis Tsiotras$^{1}$%

\thanks{$^{1}$Ji Yin and Panagiotis Tsiotras are with D. Guggenheim School of Aerospace Engineering, Georgia Institute of Technology, Atlanta, GA.
        Email:\footnotesize \texttt{\{jyin81, tsiotras\}@gatech.edu}}%
\thanks{$^{2}$Oswin So, Eric Yang Yu and Chuchu Fan are with the Department of Aeronautics and Astronautics, Massachusetts Institute of Technology, Cambridge, MA. 
        {\footnotesize Email:\texttt{\{oswinso, eyyu, chuchu\}@mit.edu}}\vspace{.6ex}}%
\thanks{$*$ Equal contribution}
}

\begin{document}

\maketitle
\thispagestyle{plain}
\pagestyle{plain}

\begin{abstract}

A common problem when using model predictive control (MPC) in practice is the satisfaction of safety specifications beyond the prediction horizon.
While theoretical works have shown that safety can be guaranteed by enforcing a suitable terminal set constraint or a sufficiently long prediction horizon, these techniques are difficult to apply and thus are rarely used by practitioners, especially in the case of general nonlinear dynamics.
To solve this problem, we impose a tradeoff between exact recursive feasibility, computational tractability, and applicability to ``black-box'' dynamics by learning an approximate discrete-time control barrier function and
incorporating it into a variational inference MPC (VIMPC), a sampling-based MPC paradigm.
To handle the resulting state constraints, we further propose a new sampling strategy that greatly reduces the variance of the estimated optimal control, improving the sample efficiency, and enabling real-time planning on a CPU.
The resulting Neural Shield-VIMPC (NS-VIMPC) controller yields substantial safety improvements compared to existing sampling-based MPC controllers, even under badly designed cost functions.
We validate our approach in both simulation and real-world hardware experiments.
Project website: \url{https://mit-realm.github.io/ns-vimpc/}.

\end{abstract}

\section{Introduction}

Model Predictive Control (MPC) is a versatile control approach
widely used in robotics applications such as autonomous driving \cite{InfoMPPI}, bio-inspired locomotion \cite{brasseur2015robust, snakegait},  or manipulation of deformable objects \cite{power2021keep}, to name just a few.
These methods address safety by incorporating state and control constraints into the finite-horizon optimization problem, ensuring that the system remains safe over the prediction horizon \cite{brasseur2015robust,lindqvist2020nonlinear,power2021keep,wang2021variational}.
However, safety of the system beyond the prediction horizon is often overlooked by practitioners, potentially leading to the violations of safety constraints at future timesteps.

This is a well-known problem in the field of MPC.
The question of whether a sequence of safe control actions can always be found under an MPC controller has been studied extensively in the literature under the name of recursive feasibility \cite{kerrigan2000invariant,mayne2000constrained,chen2003terminal,gondhalekar2009controlled,lofberg2012oops}.
A simple method of achieving recursive feasibility is by enforcing a control-invariant terminal set constraint at the end of the prediction horizon \cite{kerrigan2000invariant,chen2003terminal}.
However, it is difficult to find such a control-invariant set for general nonlinear systems.
Methods such as bounding the system dynamics \cite{chen2003terminal,bravo2005computation,fiacchini2007computation} often result in over-conservative sets, while methods that numerically solve the Hamilton-Jacobi (HJ) PDE \cite{margellos2011hamilton} scale exponentially with the number of state variables, and hence this approach is computationally infeasible for systems with more than, say, $5$ state variables~\cite{mitchell2008flexible}. 

Even without considering recursive feasibility, nonlinear \textit{constrained} optimization is a challenging problem.
Traditionally, constraints are handled using techniques such as interior-point, sequential quadratic programming and augmented Lagrangian type methods \cite{nocedal1999numerical}, which mostly rely on accurate linear or quadratic approximations of the cost and constraints.
However, these gradient-based methods can get stuck in local minima or fail to converge when the problem is highly nonlinear.
Moreover, first and second-order gradient information is needed to solve these optimization problems, which is often not available for ``black-box'' systems.
Consequently, MPC controllers that rely on these underlying nonlinear constrained optimizers \cite{nocedal1999numerical,jallet2022constrained,gill2005snopt} inherit the same limitations.

\begin{figure}[t]
    \centering
    \includegraphics[width=0.8\linewidth]{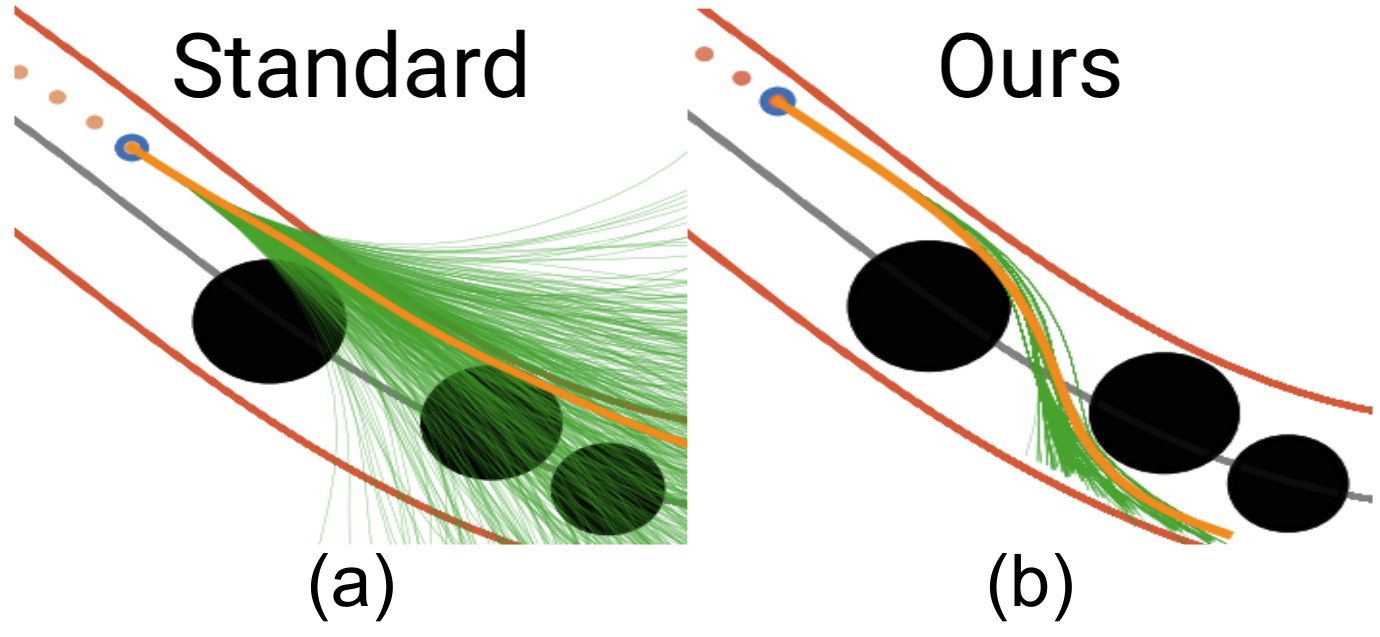}
    \caption{\textbf{Standard sampling vs. Resampling-Based Rollout (RBR).}
    In this \textsf{AutoRally} example, the blue dot samples future trajectories and computes an optimal control that avoids the black obstacles.
    RBR rewires all sampled trajectories to be safe, resulting in a more accurate sampling distribution. In contrast, the standard approach samples from a Gaussian distribution and wastes computation on unsafe trajectories.
    \label{fig:EfficientSamplingVsNormalSampling1}}
\end{figure}

\begin{figure*}[t]
    \centering
    \makebox[\textwidth][c]{\includegraphics[width=1.02\linewidth]{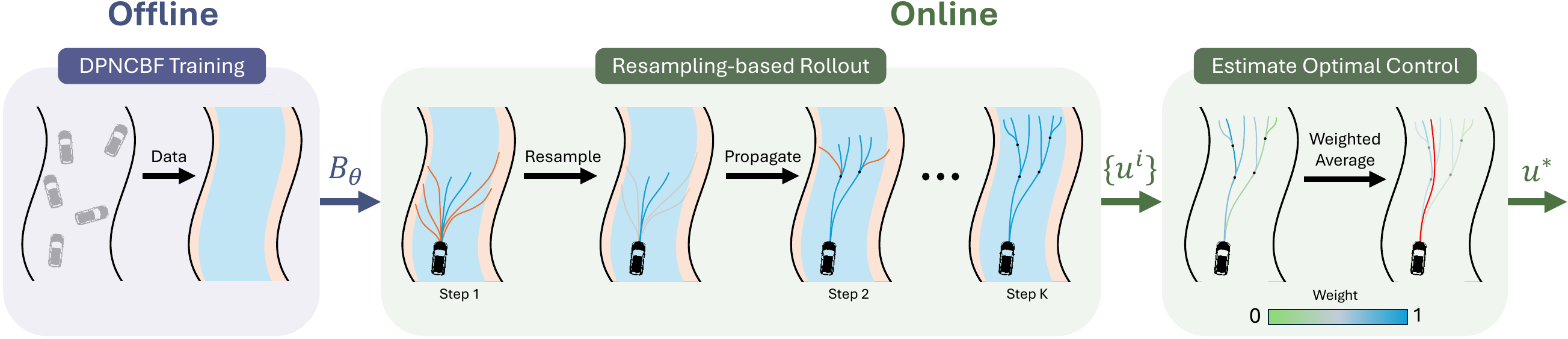}}%
    \caption{
    \textbf{Overview of Neural Shield VIMPC (NS-VIMPC).}
    Offline, using DPNCBFs, we collect a dataset and train a NN approximation $B_\theta$ of a DCBF.
    $B_\theta$ is used to impose the DCBF descent condition \eqref{eqn:descent_condition2} as a state constraint.
    Online, we modify the VIMPC architecture to use resampling-based rollouts to improve the sampling distribution of the Monte Carlo estimator in the presence of the DCBF state constraints.
    In addition, we integrate the DCBF safety condition \eqref{eqn:descent_condition} into the optimization objective function as in \eqref{eq:cbfpenalty}.
    }
    \label{fig:control_architecture}
\end{figure*}

One way to address these previous limitations is to use sampling-based optimization methods,
which have become popular within both the robotics and model-based reinforcement learning communities.
In particular, variational inference MPC (VIMPC)
\cite{okada2020variational,lambert2020stein,wang2021variational,barcelos2021dual,power2022variational,power2024learning}
has emerged as a popular family of methods that pose the control problem as an inference problem instead of an optimization problem through the control as inference framework \cite{attias2003planning,toussaint2006probabilistic,rawlik2013stochastic}, and perform variational inference to approximate the resulting intractable optimal control distribution.
Popular MPC controllers within this family include Model Predictive Path Integral (MPPI) \cite{InfoMPPI} and the Cross-Entropy Method (CEM) \cite{rubinstein1999cross,wang2021variational},
which have been applied to robotics tasks such as autonomous driving \cite{InfoMPPI}, bipedal locomotion \cite{brasseur2015robust}, manipulation of deformable objects \cite{power2021keep}, and in-hand manipulation \cite{nagabandi2020deep}, among many others.
Using this framework, constraints can be easily handled by appropriately manipulating the posterior control distribution without facing the same optimization challenges as traditional gradient-based methods. 
However, the problem of safety beyond the prediction horizon still remains a challenge and has not been addressed by existing works that employ VIMPC.

In this work, we present a novel sampling-based MPC approach that provides safety beyond the prediction horizon by using control barrier functions to enforce the control-invariant set constraint for VIMPC controllers.
To tackle the challenge of finding control-invariant sets for general nonlinear systems, we extend our previous work on learning neural network approximations of control barrier functions using policy neural control barrier functions (PNCBF) \cite{so2023train} to the discrete-time case.
Inspired by particle filtering and sequential Monte Carlo methods, we further propose a novel sampling strategy for handling state constraints in VIMPC that significantly improves the sampling efficiency and enables real-time planning on a CPU.
Results from both simulation and hardware experiments suggest that the resulting Neural Shield-VIMPC (NS-VIMPC) controller outperforms existing MPC baselines in terms of safety, even under adversarially tuned cost functions.

\noindent\textbf{Contributions.} We summarize our contributions below.
\begin{itemize}
    \item 
    We extend policy neural CBF (PNCBF) to the discrete-time case and propose a novel approach to train a discrete-time PNCBF (DPNCBF) using policy evaluation.
    
    \item 
    We propose Resampling-Based Rollout (RBR), a novel sampling strategy for handling state constraints in VIMPC inspired by particle filtering, which significantly improves the sampling efficiency by lowering the variance of the estimated optimal control.

    \item 
    Simulation results on two benchmark tasks show the efficacy of NS-VIMPC compared to existing sampling-based MPC controllers in terms of safety and sample efficiency.
    
    \item 
    Hardware experiments on AutoRally \cite{AutorallyHardware}, a $1/5$ scale autonomous driving platform, demonstrate the robustness of NS-VIMPC to unmodeled dynamical disturbances under adversarially tuned cost functions. 
\end{itemize}

\section{Related Work}

\noindent\textbf{Sampling-based MPC.} 
Sampling-based MPC has become a popular alternative to traditional MPC methods that use gradient-based solvers, in part due to the advent of parallel computing and the recent advances in GPU hardware.
As a gradient-free method, sampling-based MPC can be applied to any problem without requiring specific problem structures.
MPPI \cite{InfoMPPI} is a popular sampling-based MPC approach that formulates a variational inference problem, then solves it in the case of the Gaussian distribution, having strong connections to stochastic optimal control \cite{TubeMPPI} and maximum entropy control \cite{so2022maximum}.
Separately, the Cross-Entropy Method (CEM) \cite{rubinstein1999cross} has become popular in the reinforcement learning community \cite{wang2021variational}, in part due to its simplicity.
Both approaches were recently shown to be part of the VIMPC family of MPC algorithms~\cite{okada2020variational}. 
Consequently, some works have looked at expanding the VIMPC family to include different choices of divergences \cite{wang2021variational} and sampling distributions beyond Gaussians \cite{lambert2020stein,okada2020variational}.

\beforetextbf{}

\noindent\textbf{Safety in Sampling-based MPC.}
Recent research efforts assess the risk associated with uncertain areas in the state space during the exploration phase, and enhance MPPI's safety by incorporating a risk penalty into its cost function \cite{RAMPPI,UncertaintyPushingMPPI}. 
While these methods empirically enhance safety, they lack formal assurances. 
Another class of MPPI alternatives leverage an auxiliary tracking controller to follow the MPPI output trajectories \cite{TubeMPPI, L1Adaptive}, improving robustness against unforeseeable disruptions, but these improvements are limited when the simulation-to-reality gap is significant. 
More recently, Control Barrier Functions (CBF) \cite{ames2016control,xu2015robustness,CBFTA} have been used to provide formal safety guarantees for MPPI controllers \cite{MPPI-CBF,Shield-MPPI,yin2024chance}. 
However, these methods use distance functions as CBFs, and thus are unable to handle bounded input limits.
Our proposed Neural Shield MPPI (NS-MPPI) controller, as a specific form of the proposed NS-VIMPC framework, effectively resolves this critical issue of safety under input constraints identified above.

\beforetextbf{}

\noindent\textbf{Different Proposal Distributions for VIMPC.}
Many works investigate changing the sampling distribution of VIMPC to improve the performance of the sampling-based controller.
The MPPI variant in~\cite{CCMPPI} uses covariance steering to assign a terminal covariance to the sampling distribution, but this relies on expert knowledge on how the covariance should be designed.

Normalizing flows are used in \cite{power2022variational} to approximate the optimal sampling distribution, but this does not address the problem of recursive feasibility on its own.
Another direction looks at changing the effective sampling distribution by modifying the underlying dynamical system to be more amenable to calculations.
The works \cite{TubeMPPI, L1Adaptive} leverage an auxiliary tracking controller, thus changing the sampling distribution to be the output of the stable tracking controller.
However, this requires the construction of such an auxiliary tracking controller, which can be difficult to perform for arbitrary nonlinear discrete-time systems.
In contrast, our proposed resampling-based rollouts (RBR) can be viewed as a way of easily improving the proposal distribution \textit{without} the need for any specialized problem structure.

\beforetextbf{}

\noindent\textbf{Duality between Control and Inference.}
The proposed resampling strategy in our work is similar in spirit to \cite{xie2020factor,qadri2022incopt} in that
factor graphs, a method used originally for estimation, is adapted for control purposes.
In \cite{zhang2023optimal} a linear optimal control was used to improve the performance of particle filters.
However, to the best of our knowledge, our approach is the first to adopt the resampling mechanism from particle filters
to improve the performance of sampling-based MPC controllers.

\beforetextbf{}

\noindent\textbf{Control Barrier Functions.} 
Designing CBFs with control-invariant safe sets is not a trivial task. As a result, previous works only seek saturating CBFs that do not consider input constraints \cite{Robotarium, lindemann2018control, xu2018safe}. Some works use hand-tuned CBFs to prevent saturation \cite{wei2022safe, clark2021verification}, which can lead to overly conservative safe sets. Other works develop CBFs with input constraints for specific types of systems \cite{cortez2022safe, cortez2021robust}. The emerging neural CBFs \cite{zhang2023neural, yu2023sequential, qin2021learning, zhang2024gcbf+} allow for more general dynamics utilizing machine learning techniques; however, they can still saturate the control limits. 

\noindent\textbf{Reachability.}
Recent works in the safety community have realized that the value function of the reachability problem is a valid CBF that takes input constraints into account \cite{choi2021robust,so2023train,he2024agile}.
In continuous-time, the reachability problem can be formulated as a Hamilton-Jacobi PDE \cite{margellos2011hamilton} and is conventionally solved using grid-based numerical PDE solvers \cite{mitchell2008flexible}.
However, grid-based methods suffer from the curse of dimensionality and cannot be practically applied to systems with a state space larger than $5$ dimensions \cite{mitchell2008flexible}.
To resolve this problem, recent works have looked at using learning-based methods at solving reachability problems in both continuous-time \cite{bansal2021deepreach} and discrete-time \cite{fisac2019bridging,hsu2021safety}.
In particular, the value functions of both the \textit{avoid} \cite{choi2021robust} and \textit{reach-avoid} problems \cite{he2024agile} are CBFs.
Though early works focused on the use of \textit{optimal} value functions \cite{choi2021robust}, recent works have shown that the value function corresponding to \textit{any} policy can also serve as a CBF, which is closely related to the ideas of backup-CBFs \cite{chen2021backup}.
While this idea has been used for the avoid problem continuous-time settings \cite{so2023solving} and for the reach-avoid problem in discrete-time settings \cite{he2024agile}, it has not so far been used for the avoid problem in discrete-time, which is the approach we take in this work with Discrete-time Policy Neural Control Barrier Function (DPNCBF).
We summarize the relationship of this work with existing deep reachability-based methods in \Cref{tab:compare_existing}.

\begin{table}[]
    \vspace*{-0.5em}
    \centering
    \caption{Relationship between DPNCBF and other reachability methods.}
    \label{tab:comparison}
    \vspace{-.9em}
    \scriptsize
    \begin{tabular}{cccc}
    \toprule
       Method                       & Avoid     & Discrete  & (Arbitrary) Policy-Conditioned \\ \midrule
        \citet{bansal2021deepreach} & \cmark    & \xmark    & \xmark \\
        \citet{fisac2019bridging}   & \cmark    & \cmark    & \xmark \\
        \citet{hsu2021safety}       & \xmark    & \cmark    & \xmark \\
        \citet{so2023train}         & \cmark    & \xmark    & \cmark \\
        \citet{he2024agile}         & \xmark    & \cmark    & \cmark \\ \midrule
        \textbf{DPNCBF (ours)}      & \cmark    & \cmark    & \cmark \\
     \bottomrule
    \end{tabular}
    \label{tab:compare_existing}
    \vspace{-0.5em}
\end{table}

\section{Problem Formulation}

We consider the discrete-time, nonlinear dynamics
\begin{align}\label{dynamics}
    x_{k+1} = f(x_k, u_k),
\end{align}
with state $x\in \mathcal{X} \subseteq \mathbb{R}^{n_x}$ and control $u\in \mathcal{U} \subseteq \mathbb{R}^{n_u}$.
Let $\mathcal{U}^K$ denote the set of control trajectories of length $K$.
Following the MPC setup, we assume that a cost function $J : \mathcal{U}^K \to \mathbb{R}$ encoding the desired behavior of the system is given.
Moreover, we consider state constraints defined by an 
\textit{avoid set} $\mathcal{A} \subset \mathcal{X}$ described as the superlevel set of some specification function $h$, i.e., 
\begin{equation} \label{eqn:AvoidSet}
    \mathcal{A} \coloneqq \{x \in \mathcal{X} \mid h(x) > 0\}.
\end{equation}
The goal is then to find a sequence of controls $\vu =
\{u_0,u_1,\ldots,u_{K-1}\} \in \mathcal{U}^K$ that minimizes the cost function $J$ while satisfying the system dynamics \eqref{dynamics} and safety constraints $x_k \not\in \mathcal{A}$ 
for all $k \geq 0$.

\subsection{Variational Inference MPC For Sampling-based Optimization}

To solve the above optimization problem, we deviate from traditional MPC solvers and use a variational inference MPC formulation to solve the problem via sampling-based optimization.
To this end, we make use of the control-as-inference framework \cite{levine2018reinforcement} to model the problem.
Specifically, let $o$ be a binary variable that indicates ``optimality'' such that, for all controls $\vu \in \mathcal{U}^K$,
\begin{equation} \label{eq:VI:opt_cond}
    p(o=1 \mid \vu) \propto \exp(-J(\vu)),
\end{equation}
We assume a prior $p_0(\vu)$ on the control trajectory. The ``optimal'' distribution can then be obtained via the posterior distribution
\begin{equation} \label{eq:VI:posterior}
    p(\vu \mid o=1) = \frac{p(o=1 \mid \vu)p_0(\vu)}{p(o=1)} = Z^{-1} \exp( -J( \vu )) p_0(\vu),
\end{equation}
where $Z := \int \exp(-J(\vu)) p_0(\vu) \diff{\vu}$ denotes the unknown normalization constant.
Since sampling from $p(\vu \mid o=1)$ is intractable, we use variational inference to approximate the posterior $p(\vu \mid o=1)$ with a tractable distribution $q_\vv(\vu)$ parametrized by some vector $\vv$ by minimizing the forward KL divergence, i.e.,
\begin{equation} \label{eq:VI}
    \min_{\vv} \quad \KL{ p(\vu \mid o=1) }{ q_\vv(\vu) }.
\end{equation}
In the special case of $q_\vv$ being a Gaussian distribution with mean $\vv$ and a fixed control input covariance $\Sigma$, intrinsic to the robotic system of interest \cite{InfoMPPI},
we can solve \eqref{eq:VI} in closed-form to obtain the optimal $\vv^*$ as (see \Cref{app:VI Gauss} for details)
\begin{equation} \label{eq:vi:v_opt_def}
    \vv^* = \E_{p(\vu \mid o=1)}[ \vu ].
\end{equation}
While this expectation cannot be readily computed because $p(\vu \mid o=1)$ is intractable to sample from,
we can use importance sampling to change the sampling distribution to some other distribution $r(\vu)$ that is easier to sample from, leading to
\begin{align}
    \vv^*
    &= \E_{r(\vu)}\left[ \frac{ p(\vu \mid o=1)}{r(\vu)} \vu \right] \\
    &= \E_{r(\vu)}\bigg[ \underbrace{\frac{Z^{-1} \exp(-J(\vu)) p_0(\vu)}{r(\vu)}}_{\coloneqq \omega(\vu)} \vu \bigg]. \label{eq:vi:omega_def}
\end{align}
Using samples $\vu^1, \dots, \vu^N$ drawn from $r$, we compute a Monte Carlo estimate $\hat{\vv}$ of the optimal control sequence $\vv^*$ (see \Cref{app:VI SNIS} for details) as follows,
\begin{align}
    \hat{\vv} &= \sum_{i=1}^N \tilde{\omega}^i \vu^i, \label{eq:vi:v_opt} \\
    \tilde{\omega}^i &\coloneqq \frac{\omega(\vu^i)}{\sum_{j=1}^N \omega(\vu^j)} . \label{eq:vi:omega_tilde}
\end{align}
Note that the $Z$ in $\omega(\vu)$ is canceled out in the computation of $\tilde{\omega}^i$ in \eqref{eq:vi:omega_tilde} and hence can be ignored.

\begin{mdframed}[style=ThmFrame]
\begin{remark}[Self-normalized importance sampling]
    Note that the weights $\tilde{\omega}^i$ in \eqref{eq:vi:v_opt} are \textbf{not} from the regular importance sampling estimates in \eqref{eq:vi:omega_def} due to the normalization by the sum of the weights in \eqref{eq:vi:omega_tilde}.
    Instead, \eqref{eq:vi:v_opt} is a self-normalized importance sampling estimator (SNIS), which uses an \textit{estimate} of the weights $\omega$ but results in a biased, though asymptotically unbiased, estimator.
    This fact is not present in many existing works on both VIMPC (e.g., \cite{okada2020variational,wang2021variational}) and MPPI (e.g., \cite{InfoMPPI}).
    See \Cref{app:VI SNIS} for more details.
\end{remark}
\end{mdframed}

\begin{mdframed}[style=ThmFrame]
\begin{remark}[Connections to MPPI and CEM]
    In the case where we choose $p_0(\vu) = q_{\bf 0}(\vu)$ and $r(\vu)=q_{\bar{\vv}}(\vu)$ for some previous estimate of the optimal control sequence $\bar{\vv}$, the above variational inference MPC framework reduces to MPPI (see \Cref{app:MPPI as VI} for details).
    Moreover, Cross-Entropy Method (CEM) also falls in the VIMPC framework \cite{okada2020variational,wang2021variational}.
\end{remark}
\end{mdframed}

\subsection{Constraint Handling In Variational Inference MPC} \label{subsec:constr_handling}
One advantage of sampling-based MPC is that it is simple to incorporate hard constraints. 
One can include an indicator function in the cost function that heavily penalizes constraint violations (e.g., see \cite{InfoMPPI,TubeMPPI,barcelos2021dual,bhardwaj2022storm}). Specifically, for some large constant $C > 0$, we can modify the cost function as
\begin{equation} \label{eq:constr_handle}
    J_{\text{new}}(\vu) = J(\vu) + C \sum_{k=0}^{K} \ind{x_k \in \mathcal{A}}.
\end{equation}
From the inference perspective \eqref{eq:VI:opt_cond}, we can interpret $J_{\text{new}}$ \eqref{eq:constr_handle} as saying that $o = o_{\text{cost}} \land o_{\text{constraint}}$, where,
\begin{align}
    p(o_{\text{cost}} = 1 \mid \vu) &\propto \exp( -J(\vu) ), \\
    p(o_{\text{constraint}} = 1 \mid \vu) &\propto \exp\left( -C \sum_{k=0}^{K} \ind{x_k \in \mathcal{A}} \right). \label{eq:constr_handle:p_constr}
\end{align}
However, a problem with \eqref{eq:constr_handle:p_constr} is that it looks at safety only within the prediction horizon and does not consider the probability of staying safe beyond the prediction horizon. 
To tackle this problem, we will use a discrete-time control barrier function (DCBF), which we introduce in the next section, to enforce that the states remain within a control-invariant set.

\subsection{Discrete-time Control Barrier Functions (DCBF)} \label{sec:DCBFDefinition}
A discrete-time control barrier function (DCBF) \cite{Shield-MPPI} associated to the avoid set $\mathcal{A}$ is a function $B : \mathcal{X} \to \mathbb{R}$ such that\footnote{Note that we use the opposite sign convention as compared to \cite{Shield-MPPI}.}
\begin{subequations}\label{eqn:DCBFDefinition}
\begin{align}
\hspace{-.3em}B(x) &> 0, \quad \forall x \in \mathcal{A},\label{eqn:avoid_condition}\\
\hspace{-.3em}B(x) &\leq 0 \implies \inf_{u\in \mathcal{U}} B(f(x,u)) - B(x) \leq -\alpha (B(x)), \label{eqn:descent_condition}
\end{align}
\end{subequations}
where $\alpha$ is an extended class-$\kappa$ function \cite{xu2015robustness}.
As in \cite{Shield-MPPI}, we restrict our attention to the class of linear extended class-$\kappa$ functions, i.e.,
\begin{align}\label{eqn:KappaFunction}
\alpha(B(x)) = a \cdot B(x), \quad a \in (0, 1).    
\end{align}
The following theorem from \cite{Shield-MPPI} proves that a controller satisfying the condition \eqref{eqn:descent_condition} renders the sublevel set $\mathcal{S} = \{x \mid B(x)\leq 0\}$ forward-invariant.

\begin{theorem}[\protect{\cite[Property 3.1]{Shield-MPPI}}] \label{thm:forward_invariant}
    Any control policy  $\pi:\mathcal{X} \rightarrow \mathcal{U}$ satisfying the condition
    \begin{equation} \label{eqn:descent_condition2}
        B(f(x,\pi(x))) - B(x) \leq -\alpha (B(x)),
    \end{equation}
    renders the sublevel set $\mathcal{S} = \{x \mid B(x)\leq 0\}$ forward-invariant.
\end{theorem}
Hence, one way to guarantee recursive feasibility, and thus safety beyond the prediction horizon, is to enforce the condition \eqref{eqn:descent_condition2} at every time step of the optimization problem.
Another point to note is that $\mathcal{S}$ is a control-invariant set, and thus the constraint $B(x_K) \leq 0$ can be imposed as a terminal state constraint to guarantee recursive feasibility for MPC as is done classically \cite{kerrigan2000invariant,chen2003terminal}.
Thus, one can try to enforce these constraints by incorporating them into the cost function as in
\eqref{eq:constr_handle} \cite{Shield-MPPI}, i.e., modify the cost according to one of the following options,
\begin{align}
    J_{\text{new}}(\vu) &= J(\vu) + C \sum_{k=0}^{K} \ind{ B(f(x,u)) - B(x) > -\alpha (B(x)) }, \\
    J_{\text{new}}(\vu) &= J(\vu) + C \sum_{k=0}^{K} \Big[ B(f(x,u)) - B(x) + \alpha( B(x) ) \Big]_+. \label{eq:cbfpenalty}
\end{align}
However, this approach has two problems. 
First, for sufficiently large $C$, samples that violate the DCBF constraint will have a normalized weight of near zero, rendering these samples useless.
Conservative DCBFs may cause the majority of the samples to violate the constraint \eqref{eqn:descent_condition2} and hence have zero weight, resulting in poor estimates of the optimal control and wasted computation.
Second, while constructing a function $B$ that satisfies \eqref{eqn:avoid_condition} is relatively simple, it is much harder to construct a function $B$ that also satisfies \eqref{eqn:descent_condition}, contrary to the case with (continuous-time) control barrier functions \cite{so2023train,yu2023sequential} (see also \Cref{Remark3} below).
Consequently, many works that integrate control barrier functions into MPC often only propose functions for which \eqref{eqn:avoid_condition} holds and not \eqref{eqn:descent_condition} \cite{Shield-MPPI}, rendering the safety guarantees of \Cref{thm:forward_invariant} invalid.

In the next section, we address these two problems
via a novel resampling method that reuses computations from zero-weight samples and learns a DCBF that tries to respect both \eqref{eqn:avoid_condition} and \eqref{eqn:descent_condition}  using policy value functions.
\begin{mdframed}[style=ThmFrame]
\begin{remark}[Differences between discrete-time and continuous-time control barrier functions under unbounded controls]\label{Remark3}
    Note that in the continuous-time case where having an unbounded control space $\mathcal{U}$, control-affine dynamics, and a non-zero $\frac{\partial \dot{B}}{\partial u}$ are sufficient to guarantee that $B$ is a valid control barrier function. This is because \eqref{eqn:descent_condition} generally holds for a function $B$ that satisfies \eqref{eqn:avoid_condition} in the continuous-time case under these assumptions.
    However, the same does not apply to discrete-time under general nonlinear dynamics $f$ since \eqref{eqn:descent_condition2} is nonlinear in $u$, let alone the fact that robotic systems in real life are unable to exert infinite forces and hence generally do not have unbounded controls.
\end{remark}
\end{mdframed}

\begin{mdframed}[style=ThmFrame]
\begin{remark}[Constraint satisfaction in the variational inference framework] \label{rem:convex_thing}
    One potential issue with incorporating DCBF constraints into the cost function $J$ 
    is that,
    while $p(\vu \mid o=1)$ may have zero density on the set of controls that violate the DCBF constraint,
    this does not necessarily hold for the approximating distribution $q$ since we are minimizing the \textit{forward} KL divergence \cite{jerfel2021variational}.
    In particular, the mean $\vv$ of $q_\vv$, which is the control to be used, may not satisfy the DCBF constraint.
    One way to guarantee that $\vv$ does satisfy the DCBF constraint is to assume that the set of controls that satisfy the DCBF constraint is itself convex (see \Cref{app:convex_controls}).
    However, this is an unrealistic assumption that is often violated by state constraints such as obstacle avoidance.
    Despite these shortcomings, we observed in our experimental results that this method of enforcing DCBF constraints on $p(\vu \mid o=1)$ indeed drastically improved safety.
    Alternatively, this problem can be solved by checking for constraint satisfaction of the mean $\vv$, and if not satisfied, replacing it with any of the rollouts that do satisfy the constraints, similar to \cite{borquez2025dualguard}. 
    We leave further theoretical exploration of this issue as future work.
\end{remark}
\end{mdframed}

\section{Neural Shield VIMPC}\label{sec:NSMPPI}

In this section, we propose Neural Shield VIMPC (NS-VIMPC),
a sampling-based MPC paradigm that efficiently samples trajectories using a DCBF modeled using a neural network.
We illustrate the proposed NS-VIMPC algorithm in \Cref{fig:control_architecture}.

\subsection{Approximating DCBF Using Neural Policy Value Functions}\label{sec:ConstructingCBF}

Let $x^\pi_{k}$ denote the state at time $k$
following the control policy $\pi : \mathcal{X} \to \mathcal{U}$. Define the policy value function $V^{h, \pi}$ as,
\begin{align}\label{eqn:ValueFunctionDefinition}
    V^{h, \pi}(x_0) \coloneqq \max_{k \geq 0} h(x^\pi_{k}).
\end{align}
We then have the following theorem.
\begin{theorem}
    $V^{h, \pi}$ satisfies \eqref{eqn:avoid_condition} and \eqref{eqn:descent_condition} and is a DCBF.
\end{theorem}
\begin{proof}
From the definition of $V^{h,\pi}$ \eqref{eqn:ValueFunctionDefinition}, we have that
\begin{align}
    V^{h, \pi}(x_k) &\geq h(x_k), \label{eqn:ValueFunctionProperty1}\\
    V^{h, \pi}(x_k) &\geq V^{h, \pi}(f(x_k, \pi(x_k))). \label{eqn:ValueFunctionProperty2}
\end{align}
Using the definition of the avoid set $\mathcal{A}$ \eqref{eqn:AvoidSet} and \eqref{eqn:ValueFunctionProperty1}, it follows that $V^{h,\pi}(x) > 0$ for all $x \in \mathcal{A}$, satisfying the first condition \eqref{eqn:avoid_condition} of a DCBF.
When $V^{h,\pi}(x_k) \leq 0$, we have $-\alpha(V^{h,\pi}(x_k)) \geq 0$, and  \eqref{eqn:ValueFunctionProperty2} implies that,
\begin{equation}
    V^{h,\pi}(f(x_{k}, \pi(x_k))) - V^{h,\pi}(x_k) \leq 0 \leq -\alpha(V^{h,\pi}(x_k)).
\end{equation}
Since $\pi(x_k) \in \mathcal{U}$, this implies the second condition \eqref{eqn:descent_condition}.
Thus, the policy value function $V^{h,\pi}$ is a DCBF.
\end{proof}

Although we have constructed a DCBF from \eqref{eqn:ValueFunctionDefinition}, the challenge is that the policy value function $V^{h,\pi}$ cannot be easily evaluated at arbitrary states
since the maximization in \eqref{eqn:ValueFunctionDefinition} is taken over an infinite horizon.
To fix this, we train a neural network approximation $V_\theta^{h,\pi}$ of $V^{h, \pi}$, extending the approach of \cite{so2023train} to the discrete-time case.
To begin, we first rewrite \eqref{eqn:ValueFunctionDefinition} in a dynamic programming form,
\begin{equation} \label{eq:value_fn_dp}
    V^{h,\pi}(x_0) = \max\Big\{ \max_{0 \leq k \leq T} h(x_{k}^\pi),\; V^{h, \pi} (x^\pi_{T}) \Big\}.
\end{equation}
We can then train a neural network $V_\theta^{h,\pi}$ to approximate the value function $V^{h,\pi}$ by minimizing the loss
\begin{equation}\label{eqn:DPcost}
    L(\theta) = \norm{V^{h,\pi}_\theta (x_0) - \max \biggl\{ \max_{0 \leq k \leq T} h(x_{k}^\pi), V^{h, \pi}_\theta (x^\pi_{T})\biggr\} }^2,
\end{equation}
over all states $x_0$.
One problem, however, is that the minimizer of \eqref{eqn:DPcost} is not unique.
For example, if $h(x) \leq \bar{h}$ for all $x$, then $V_\theta^{h,\pi} = \bar{h}$ is a minimizer of \eqref{eqn:DPcost} but does not necessarily satisfy \eqref{eqn:ValueFunctionDefinition}.
To fix this, we follow the approach of \cite{fisac2019bridging,so2023solving} and, inspired by reinforcement learning \cite{sutton2018reinforcement}, introduce a discount factor $\gamma \in (0,1)$ to define the \textit{discounted} value function
\begin{equation} \label{eq:value_fn_dp_discounted}
    V^{h,\pi,\gamma}(x_k)
    = \max \Big\{ h(x_k^\pi),\; (1 -\gamma) h(x_k) + \gamma V^{h, \pi, \gamma} (x_{k+1}^\pi) \Big\},
\end{equation}
and the corresponding loss $L$,
\begin{align}\label{eqn:DPcostModified}
    L(\theta) &= \norm{V^{h,\pi,\gamma}_\theta (x_k) - \hat{V}^{h,\pi,\gamma}_\theta(x_k) }^2, \\
    \hat{V}^{h,\pi,\gamma}_\theta(x_k) &= \max \biggl\{ h(x_k^\pi), (1-\gamma) h(x_k^\pi) + \gamma V^{h, \pi,\gamma}_\theta (x^\pi_{k+1})\biggr\}.
\end{align}
Finally, instead of using the learned $V^{h,\pi,\gamma}_\theta$ directly in \eqref{eqn:descent_condition2}, we first take the maximum with $h$ and use $\tilde{V}^{h,\pi,\gamma}_\theta$ defined as
\begin{equation}
   \tilde{V}^{h,\pi,\gamma}_\theta(x) \coloneqq \max\{ h(x), V^{h,\pi,\gamma}_\theta(x) \}.
\end{equation}
This guarantees that $\tilde{V}^{h,\pi,\gamma}_\theta(x) \geq h(x)$
and hence the zero sublevel set of $\tilde{V}^{h,\pi,\gamma}$ will be a subset of $h$.
Hence, imposing the state constraint $\tilde{V}^{h,\pi,\gamma}(x) \leq 0$ will, at the very least, prevent violations of the original state constraints during the prediction horizon,
and potentially also induce a state constraint that is closer to the true control-invariant set than the original sublevel set of $h$.

\begin{mdframed}[style=ThmFrame]
\begin{remark}[Neural Network Verification of DCBFs]
    We emphasize that our goal here is to obtain a good approximation of a DCBF $B_\theta$ using a neural policy value function $V^{h,\pi,\gamma}_\theta$, and not necessarily to obtain a true DCBF.
    Verifying whether the learned $V^{h,\pi,\gamma}_\theta$ is a true DCBF requires neural network verification which can be intractable or inconclusive (see Appendix 1 in \cite{katz2017reluplex} on the NP-completeness of the NN-verification problem).
    This is especially true in the discrete-time case where the condition \eqref{eqn:descent_condition2} may not be affine in the control $u$.
\end{remark}
\end{mdframed}

Nevertheless, as we show later, empirical results show that using an approximation of a DCBF is sufficient for enabling the use of much shorter prediction horizons without sacrificing safety.

\subsection{Efficient Sampling Using Resampling-based Rollouts}\label{sec:Efficient_Trajectory_Sampling}

We tackle the problem of wasted samples with zero weights by drawing inspiration from the sequential Monte Carlo \cite{doucet2001introduction} and particle filter \cite{gordon1993novel} literature, and by performing a per-timestep resampling during the rollout, which we call Resampling-Based Rollouts (RBR).
Specifically, the control-as-inference problem formulation \eqref{eq:VI:opt_cond} gives us a \textit{temporal} decomposition of $p(o_{\text{constraint}}=1 \mid \vu)$ \eqref{eq:constr_handle:p_constr} in the case of state constraints when $C \to \infty$, as follows
\begin{equation}
    p(o_{\text{constraint}}=1 \mid \vu)
    \propto \prod_{k=0}^{K} \ind{x_k \not\in \mathcal{A}}.
\end{equation}
Hence, we can write the posterior $p(\vu \mid o=1)$ as
\begin{equation}
    p(\vu \mid o=1)
    \propto 
    p(o_{\text{cost}} = 1 \mid \vu)
    p(\vu)
    \left(\prod_{k=0}^{K} \ind{x_k \not\in \mathcal{A}} \right).
\end{equation}
By treating the term on the right as a ``measurement model,'' particle filtering \cite{gordon1993novel} can be used to solve the control-as-inference problem.
The update of the particle filter weights $\hat{w}^i_k$ at time step $k$ for particle $i$ is written as
\begin{equation} \label{eq:pf_weight_update}
    \hat{w}^i_k = \hat{w}^i_{k-1} \ind{x_k \not\in \mathcal{A}}.
\end{equation}
Then, we \textit{resample} the particles with probability proportional to their weights $\hat{w}^i_k$ to obtain a new set of particles $\vu^i$.
Due to the indicator function in the weight update \eqref{eq:pf_weight_update}, the weights are either $0$ or $1$.
Assuming there exists a particle that satisfies the state constraint, applying systematic resampling \cite{kitagawa1996monte} results in ``rewiring'' particles that violate the constraint to particles that still maintain safety, reusing the computation from zero-weight samples (see \Cref{fig:resampling}).
If all particles violate the constraint, we do not resample.
In this case, since we want to minimize constraint violations, we still use the cost term \eqref{eq:cbfpenalty} such that trajectories with higher constraint violations have higher costs.
Overall, at each timestep except for the last, we check constraint satisfaction for each particle. This gives a total of $N(K-1)$ queries per rollout.

\begin{figure}[]
    \centering
    \includegraphics[width=\linewidth]{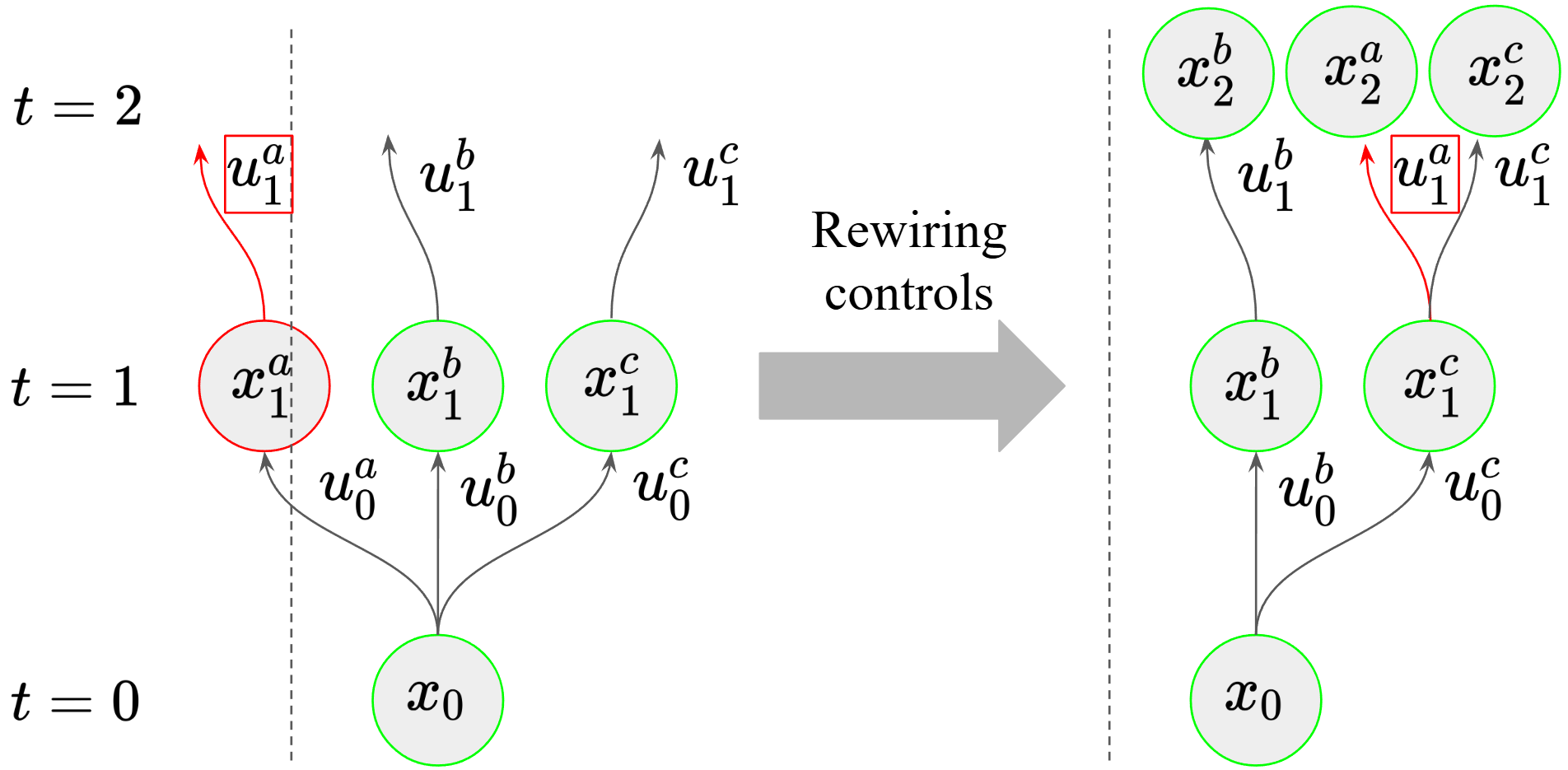}
    \caption{\textbf{Resampling-based Rollouts (RBR).} Inspired by particle filtering, at each step of the trajectory rollout, we uniformly resample any samples that violate the state constraints among the set of safe samples.
    In this example,
    the two particles at states $x^b_1$, $x^c_1$ satisfy the constraints and have weights $\hat{w}^b_1 = \hat{w}^c_1 = 1$.
    The particle at state $x^a_1$ violates constraints and thus has a weight of $\hat{w}^a_1=0$.
    Consequently, both $x_1^a$ and the control prefix $u_0^a$ are resampled away and replaced with equal probability by either $x^b_1$ or $x^c_1$ and their control prefixes respectively (in this example, $x^c_1$ was chosen). Note that $u^a_1$ can be left untouched, though is now applied from $x^c_1$ instead of the unsafe state $x^a_1$.
    }
    \label{fig:resampling}
\end{figure}

While we can prove that this resampling is unbiased from the particle filter perspective,
we also provide a more direct proof of this fact without the analogy to particle filters.
\begin{mdframed}[style=ThmFrame]
\begin{theorem} \label{thm:same_marginal}
    For a probability density function $\mathsf{f}$ and set $\mathsf{S}$, let $\mathsf{a}$ be sampled from the conditional density $\mathsf{f}(\mathsf{x} \mid \mathsf{x} \in \mathsf{S})$, and let $\mathsf{b}$ be sampled from the unconditional density $\mathsf{f}$, such that $\mathsf{a}$ and $\mathsf{b}$ are independent.
    Define the ``rewired'' random variable $\tilde{\mathsf{b}}$ to be equal to $\mathsf{b}$ if $\mathsf{b} \in \mathsf{S}$ and $\mathsf{a}$ otherwise, i.e.,
    \begin{equation}
        \tilde{\mathsf{b}} =
        \mathbbm{1}_{\mathsf{b} \in \mathsf{S}} \mathsf{b} + 
        \mathbbm{1}_{\mathsf{b} \not\in \mathsf{S}} \mathsf{a}
    \end{equation}
    Then, $\tilde{\mathsf{b}}$ is also sampled from the conditional density $\mathsf{f}(\mathsf{x} \mid \mathsf{x} \in \mathsf{S})$, such that $\tilde{\mathsf{b}}$ and $\mathsf{a}$ have the same distribution, $\tilde{\mathsf{b}} \overset{d}{=} \mathsf{a}$.
\end{theorem}
\end{mdframed}
The proof is given in \Cref{sec:proof:same_marginal}. Consequently, we can use $\tilde{\mathsf{b}}$ in a Monte Carlo estimator and still obtain unbiased estimates, as shown in the following Corollary (proof is given in \Cref{sec:proof:resample_unbiased}).
\begin{mdframed}[style=ThmFrame]
\begin{corollary} \label{thm:resample_unbiased}
    For any function $w$, the Monte Carlo estimate of $\E[w( \mathsf{x} )]$ under the conditional density $\mathsf{f}(\mathsf{x} \mid \mathsf{x} \in \mathsf{S})$ using random variables $\mathsf{a}$ and $\tilde{\mathsf{b}}$ is unbiased, i.e.,
    \begin{equation}
        \E[\frac{1}{2} w( \mathsf{a} ) + \frac{1}{2} w( \tilde{\mathsf{b}} )] = \E[w( \mathsf{x} )].
    \end{equation}
\end{corollary}
\end{mdframed}

Intuitively, resampling improves the efficiency of the Monte Carlo estimator, since most of the samples do not violate the state constraints and hence contribute to the weighted sum with non-zero weight.
We can also theoretically prove that resampling improves the variance of the resulting Monte Carlo estimator.
In the following theorem, we show how resampling reduces the \textit{exponential} growth of the variance on the prediction horizon to a constant factor in the limit as the number of samples $N$ goes to infinity.

\begin{mdframed}[style=ThmFrame]
\begin{theorem}\label{thm:resample_variance}
    Let the horizon $K > 0$, consider $\mathcal{U} = [-1, 1]$, and define the avoid set and the dynamics such that $\vu \in [0, 1]^K$ is safe, and unsafe otherwise.
    Let the prior distribution $p(\vu)$ be uniform on $[-1, 1]^K$ and let $p(o=1\mid\vu)$ be the indicator function for $\vu \in [0, 1]$ such that the posterior distribution $p(\vu \mid o=1)$ is uniform on $[0, 1]^K$, and the proposal distribution $r(\vu) = 1/2^K$ is uniform on $\mathcal{U}$.
    Then, the variance of the Monte Carlo estimator of the optimal control law
    \begin{equation}
        \hat{\vv} = \frac{1}{N} \sum_{i=1}^N \frac{ p( \vu^i \mid o=1) }{ r(\vu^i) } \vu^i,
    \end{equation}
    grows exponentially in $K$, i.e.,
    \begin{equation}
        \Var[\hat{v}_k] = \frac{1}{N}\left( \frac{1}{3} 2^K - \frac{1}{4} \right), \qquad k = 1, \dots, K.
    \end{equation}
    Using resampling, the variance is upper-bounded by,
    \begin{equation}
        \Var[\hat{v}_{k,\mathrm{resample}}] = O\left( \left(\frac{1}{1 - 2^{-N}}\right)^{K} + (1-2^{-N})^{K} \right).
    \end{equation}
\end{theorem}
\end{mdframed}
The proof of \Cref{thm:resample_variance} is given in \Cref{sec:proof:resample_variance}. 
As shown in \Cref{thm:resample_variance}, although the variance is still exponential in $K$ using resampling, the base of the exponential decreases to $1$ exponentially in the number of samples $N$. In other words, in the limit as $N \to \infty$, the variance of the estimator using the proposed RBR is bounded by a constant factor.
This novel approach reduces the variance of the Monte Carlo estimator of the optimal control law in a way that mirrors the relationship between Sequential Importance Sampling (i.e., particle filters without resampling), where the variance increases exponentially with the horizon length, and Sequential Monte Carlo (i.e., particle filters with resampling), where the (asymptotic) variance only increases linearly \cite{doucet2009tutorial}.

Another method to theoretically quantify the improvement in the variance of the estimator is via the \textit{effective sample size} ($\ESS$) \cite{doucet2009tutorial,elvira2022rethinking}.
$\ESS$ is defined as the ratio between the variance of the estimator with $N$ samples from the target and the variance of the SNIS estimator \cite{elvira2022rethinking}.
It can be interpreted as the number of samples simulated from the target pdf that would provide an estimator with variance equal to the performance of the $N$-sample SNIS estimator.
However, since the $\ESS$ is computationally intractable, the \textit{approximation} (made formal in \cite{elvira2022rethinking})
\begin{equation} \label{eq:ess_hat}
    \widehat{ESS} \coloneqq \frac{1}{\sum_{n=1}^N w_n^2},
\end{equation}
is more often used in practice.
The following theorem shows that performing RBR results in either the same or higher $\widehat{ESS}$, given certain assumptions (proof is given in \Cref{sec:proof:ess}).
\begin{mdframed}[style=ThmFrame]
\begin{theorem}\label{thm:ess}
    Let $\vw = [w_1, \dots, w_m, 0, \dots, 0]$ denote the unnormalized weight vector without resampling,
    where the last $N-m$ entries are zero due to violating the safety constraints.
    Let $\vw' = \vw + \vc = [w_1, \dots, w_m, c_{m+1}, \dots, c_N ]$ denote the unnormalized weight vector resulting from safe resampling, where $\vc = [0, \dots, 0, c_{m+1}, \dots, c_N]$.
    Let $\tilde{\vw} = \vw / \norm{\vw}_1$ and $\tilde{\vw}' = \vw' / \norm{\vw'}_1$ denote the normalized weights.
    Suppose that the weights of the resampled trajectories $\vc$ are not ``drastically larger'' than the weights of the original trajectories $\vw$, i.e.,
    \begin{equation} \label{eq: ess:assumption}
        \norm{ \vc }_1 \leq 2 \frac{N}{N-1} \frac{ \norm{\vw}_2^2 }{ \norm{\vw}_1 }.
    \end{equation}
    Then, the $\widehat{ESS}$ using RBR is no smaller than the $\widehat{ESS}$ without resampling, and is strictly greater if the inequality in \eqref{eq: ess:assumption} is strict. 
    In other words,
    \begin{equation} \label{eq: ess:wts}
        \frac{1}{\norm{ \tilde{\vw} }^2_2} \leq \frac{1}{\norm{ \tilde{\vw}' }^2_2}.
    \end{equation}
\end{theorem}
\end{mdframed}

\subsection{Summary of NS-VIMPC}
We now summarize the NS-VIMPC algorithm, shown in \Cref{fig:control_architecture}. 

Offline, we first approximate a DCBF by using the DPNCBF algorithm to learn the policy value function for a user-specified policy (\Cref{sec:ConstructingCBF}).
In our experiments, we chose this to be Shield MPPI (S-MPPI) \cite{Shield-MPPI}.
Online, we use the learned DPNCBF to enforce the DPCBF constraint \eqref{eqn:descent_condition2} to try to enforce safety beyond the prediction horizon.
During sampling, we use RBR to efficiently sample control sequences $\{ \tilde{\vu}^i \}$ from the raw samples $\{ \vu^i \}$ to satisfy the DPNCBF constraint, thus improving the sample efficiency of the Monte Carlo estimate of the optimal control.
The sampled control sequences $\{ \tilde{\vu}^i \}$ are then used to compute the estimate of the optimal control $\hat{\vv}$ using \eqref{eq:vi:v_opt}.
As in MPC fashion, we only execute the first control $\hat{\vv}_0$, and $\hat{\vv}$ is then used as the parameter vector of the sampling distribution $q_\vv$ for the next iteration.

\section{Simulations}
We first performed simulation experiments to better understand the performance of the proposed Neural Shield VIMPC (NS-VIMPC) controller.  
Although many sampling-based MPC controllers fall under the VIMPC family with different choices of the prior $p_0$ and $r$, we choose to instantiate the MPPI algorithm (see \Cref{app:VI Gauss} for details), and call the resulting controller Neural Shield-MPPI (NS-MPPI).

\begin{figure}[]
    \centering
    \includegraphics[width=0.95\linewidth]{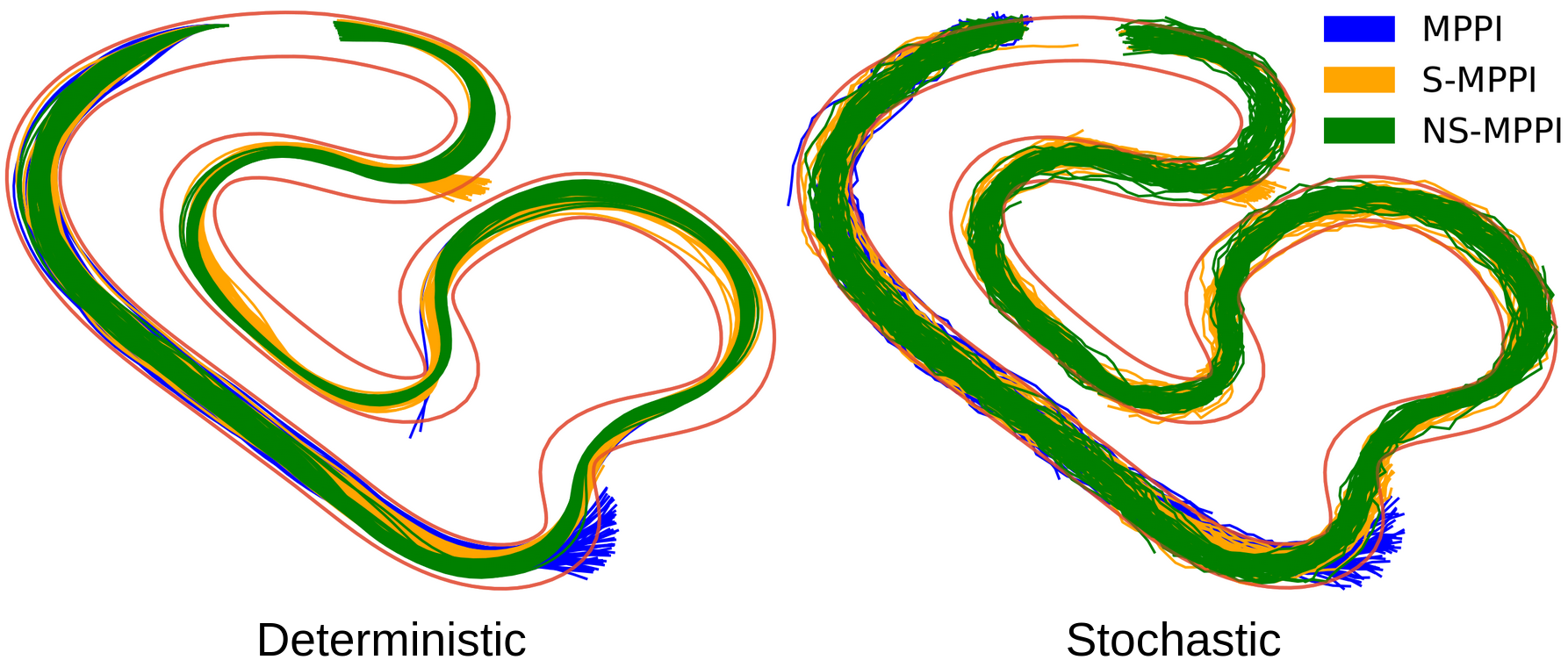}
    \caption{\textbf{\textsf{AutoRally} Trajectories.} We visualize the trajectories of the three MPPI baselines under a challenging target velocity of $\textrm{15~ms}^{-1}$.
    Both MPPI and S-MPPI veer off course and crash while NS-MPPI stays within the track even under Gaussian disturbances.
    }
    \label{fig:deter_vs_stochastic}
\end{figure}
\begin{figure}[]
    \centering
    \includegraphics[width=1.0\linewidth]{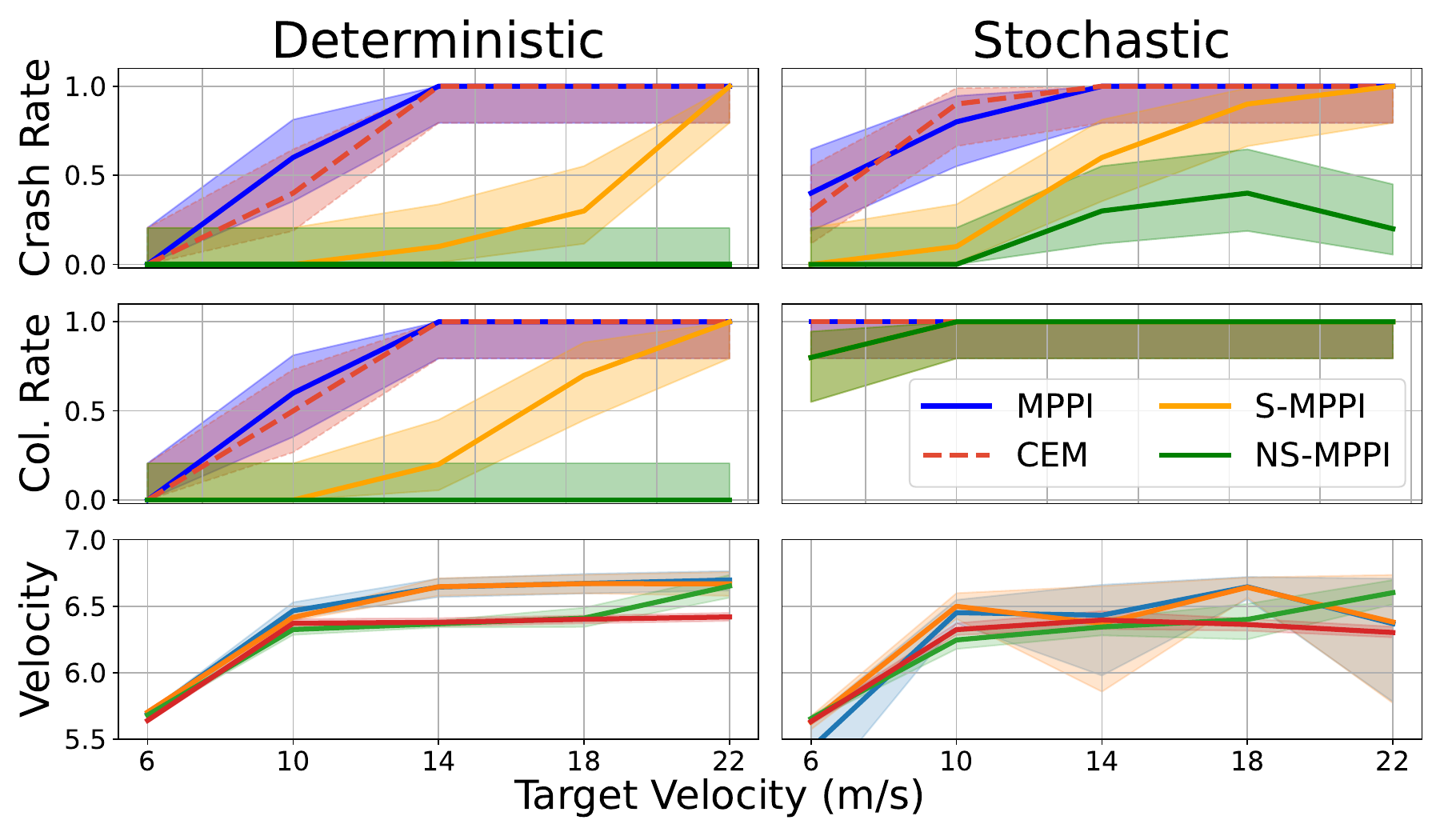}
    \caption{
    \textbf{Varying target velocities on \textsf{AutoRally}.}
    Our NS-MPPI achieves the lowest crash and collision (Col.) rates under both deterministic and stochastic dynamics.
    While the collision rate is close to $1$ for every method in the stochastic environment, NS-MPPI achieves a crash rate of near $0$.
    }
    \label{fig:deter_vs_stoch_3plot_autorally}
\end{figure}

\begin{figure}
    \centering
    \includegraphics[width=\linewidth]{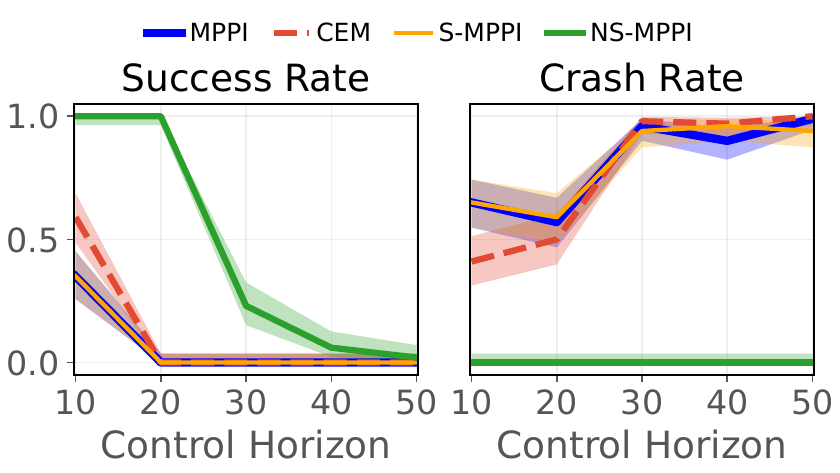}
    \caption{\textbf{Varying control horizon on \textsf{Drone}.}
    (a) Only NS-MPPI has a crash rate of zero with a control horizon of $10$, demonstrating the benefit of enforcing the DCBF constraint for maintaining safety beyond the prediction horizon.}
    \label{fig:drone_control_horizon_ablation}
\end{figure}

\beforetextbf{}
\noindent\textbf{Baseline methods.} We compared NS-MPPI against the following sampling-based MPC methods.
\begin{itemize}
\item Baseline MPPI (MPPI) ~\cite{InfoMPPI}, which forward simulates a set of randomly sampled trajectories for optimal control.

\item Shield MPPI (S-MPPI) ~\cite{Shield-MPPI}, which extends MPPI by taking $h$ in \eqref{eqn:AvoidSet} to be a DCBF and by adding the DCBF constraint violation into the cost as in \eqref{eq:cbfpenalty}.\footnote{We remove the local repair step from S-MPPI as this requires known gradients from the dynamics and would make the method gradient-based}

\item CEM \cite{pmlr-v120-bharadhwaj20a}, which samples trajectories similar to the baseline MPPI but with the weight $\omega = 1$ for only the $k$-lowest cost trajectories and $0$ otherwise. This corresponds to an average of the $k$-lowest cost trajectories.
\end{itemize}

In all simulations, we use S-MPPI as the control policy $\pi$ to learn the DPNCBF.
We provide further details on the simulation experiments in \Cref{app:sim_details}.

\subsection{Simulations on \textsf{AutoRally}}

We first compare all methods on the \textsf{AutoRally}~\cite{goldfain2019autorally} testbed, a $1/5$ scale autonomous racing car.
The goal for this task is to track a given fixed velocity without exiting the track.
The vehicle \textit{collides} when it contacts the track boundary and \textit{crashes} when it fully exceeds the track boundary.

We tested our algorithm under both deterministic dynamics and stochastic dynamics with a Gaussian state disturbance added at each timestep.
We visualize the resulting trajectories in \Cref{fig:deter_vs_stochastic} with a target velocity of $\textrm{15~ms}^{-1}$.
This is a very high target velocity, as previous works considered at most velocities of 
$\textrm{8~ms}^{-1}$ \cite{Shield-MPPI} or $\textrm{9~ms}^{-1}$ \cite{RMPPI}.
Under this challenging speed, both MPPI and S-MPPI frequently veer off course. In contrast, the proposed NS-MPPI successfully retains the vehicle within the confines of the track, thereby ensuring safety.

Next, we vary the target velocities, showing the resulting crash rate, collision rate, and velocity in \Cref{fig:deter_vs_stoch_3plot_autorally} over 20 trials.
With higher velocities, the vehicle has less time to turn, increasing the likelihood of leaving the track and colliding or crashing. 
We see that NS-MPPI consistently outperforms all other methods in both settings without having a significantly lower average velocity.

\subsection{Simulations on \textsf{Drone}}

We next tested our algorithm on \textsf{Drone}, a simulated planar quadrotor that incorporates ground effects arising from the intricate interaction between the blade airflow and the ground surface.
The goal is for the drone to navigate as fast as possible through a narrow corridor close to the ground. 
We vary the control horizon for $500$ sampled trajectories and plot the results in \Cref{fig:drone_control_horizon_ablation}. 
Only NS-MPPI has a crash rate of zero over all control horizons, while all other controllers have high crash rate in this challenging task.

To understand why this is the case, we visualize the trajectories for NS-MPPI and S-MPPI with a control horizon of $10$ in \Cref{fig:drone_viz}.
Due to the DCBF constraint, NS-MPPI starts descending to avoid the obstacles even before any of the original collision constraints are violated.
On the other hand, S-MPPI is unaware of the obstacles beyond the prediction horizon and keeps accelerating rightward until collision.
This suggests that enforcing the DCBF constraint even with an approximate $V^{h,\pi,\gamma}_\theta$ is beneficial for safety with short control horizons.

\begin{figure}
    \centering
    \includegraphics[width=.8\linewidth]{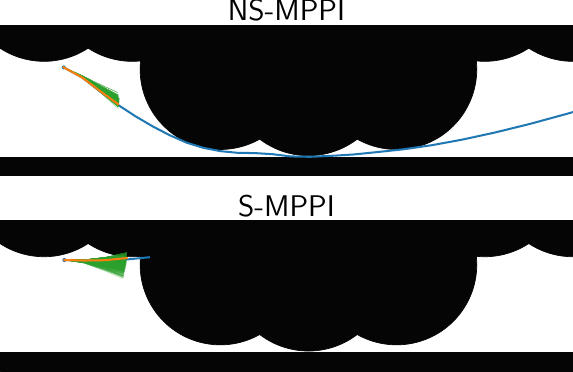}
    \caption{\textbf{Differences between NS-MPPI and S-MPPI.}
    We compare the sampling trajectory $\vu^i$ (\textcolor{tabGreen}{green}) and estimated optimal trajectory $\hat{\vv}$ (\textcolor{tabOrange}{orange}) under a control horizon of $10$.
    NS-MPPI descends early enough to avoid collisions due to the DCBF constraints despite none of the sampled trajectories violating any constraints due to the short control horizon.
    }
    \label{fig:drone_viz}
\end{figure}

\subsection{Deeper Investigation Into RBR}

We next perform various case studies to better understand the sample efficiency benefits of the proposed RBR method.

\begin{figure}
    \centering
    \includegraphics[width=0.7\linewidth]{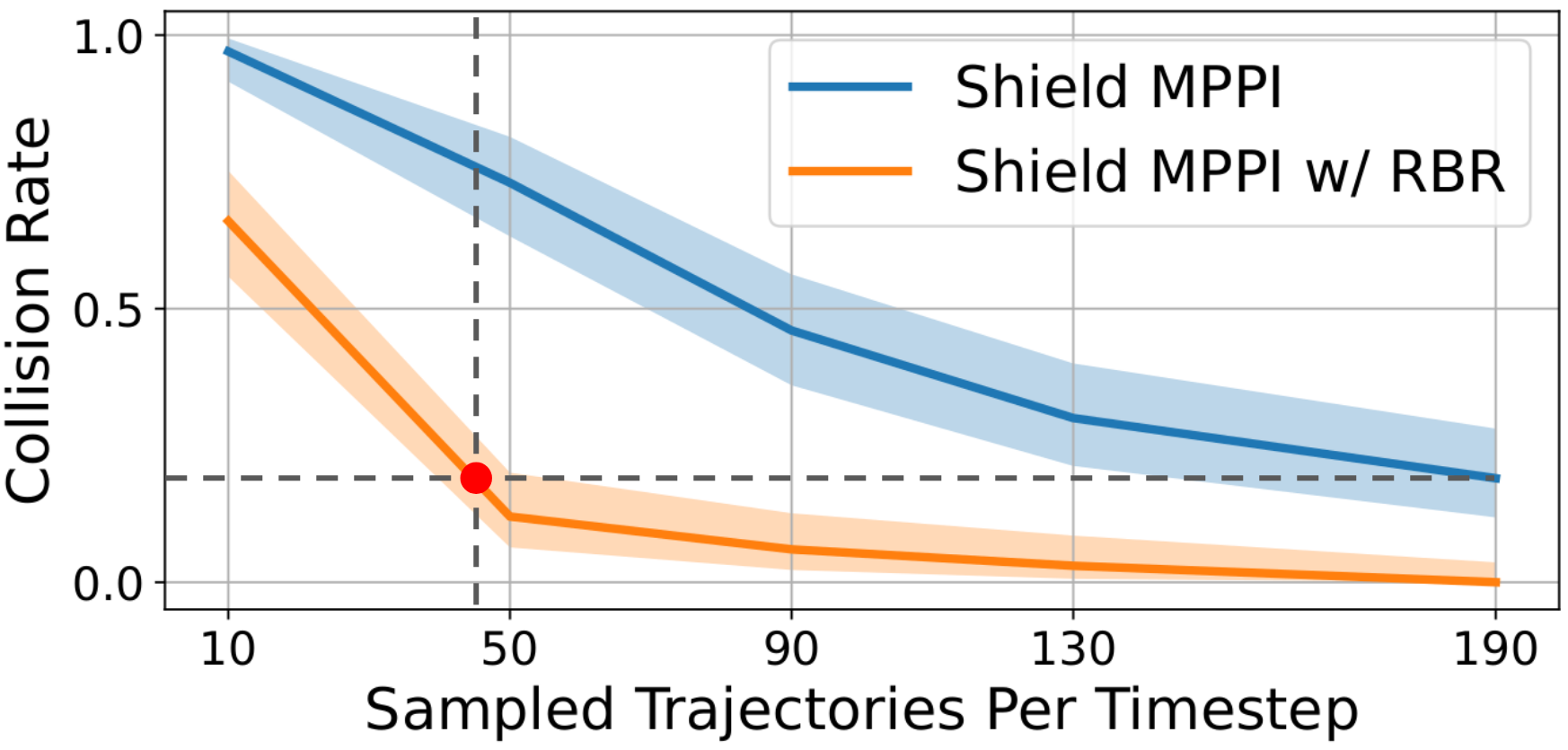}
    \setlength{\belowcaptionskip}{-8pt}
    \caption{\textbf{RBR needs 5X less samples.}
    On \textsf{AutoRally}, applying RBR to S-MPPI significantly reduces the number of sampled trajectories needed to achieve the same level of safety ($190$ to $40$, horizontal line).
    From another perspective, RBR reduces the collision rate by $74\%$ at $50$ sampled trajectories.
    }
    \label{fig:smppi_vs_smppi_efficient}
\end{figure}

\beforetextbfok{}

\noindent\textbf{RBR improves sample efficiency by 5X.}
To isolate the effects of RBR without the other improvements,
we considered a new method that extends S-MPPI with RBR (S-MPPI w/ RBR),
and compared it against the original S-MPPI across varying numbers of sampled trajectories on \textsf{AutoRally} over $100$ trials in \Cref{fig:smppi_vs_smppi_efficient}.
S-MPPI with RBR achieves the same collision rates as S-MPPI without RBR while using $5$ times \textit{fewer} trajectories.
This matches our expectations, both from the intuition that RBR results in a more closely aligned proposal distribution (e.g., see \Cref{fig:EfficientSamplingVsNormalSampling1}), and from the theoretical results in \Cref{sec:Efficient_Trajectory_Sampling} that show that RBR improves the quality of the estimator.

\begin{figure}
    \centering
    \includegraphics[width=0.9\linewidth,trim={0 5mm 0 0},clip]{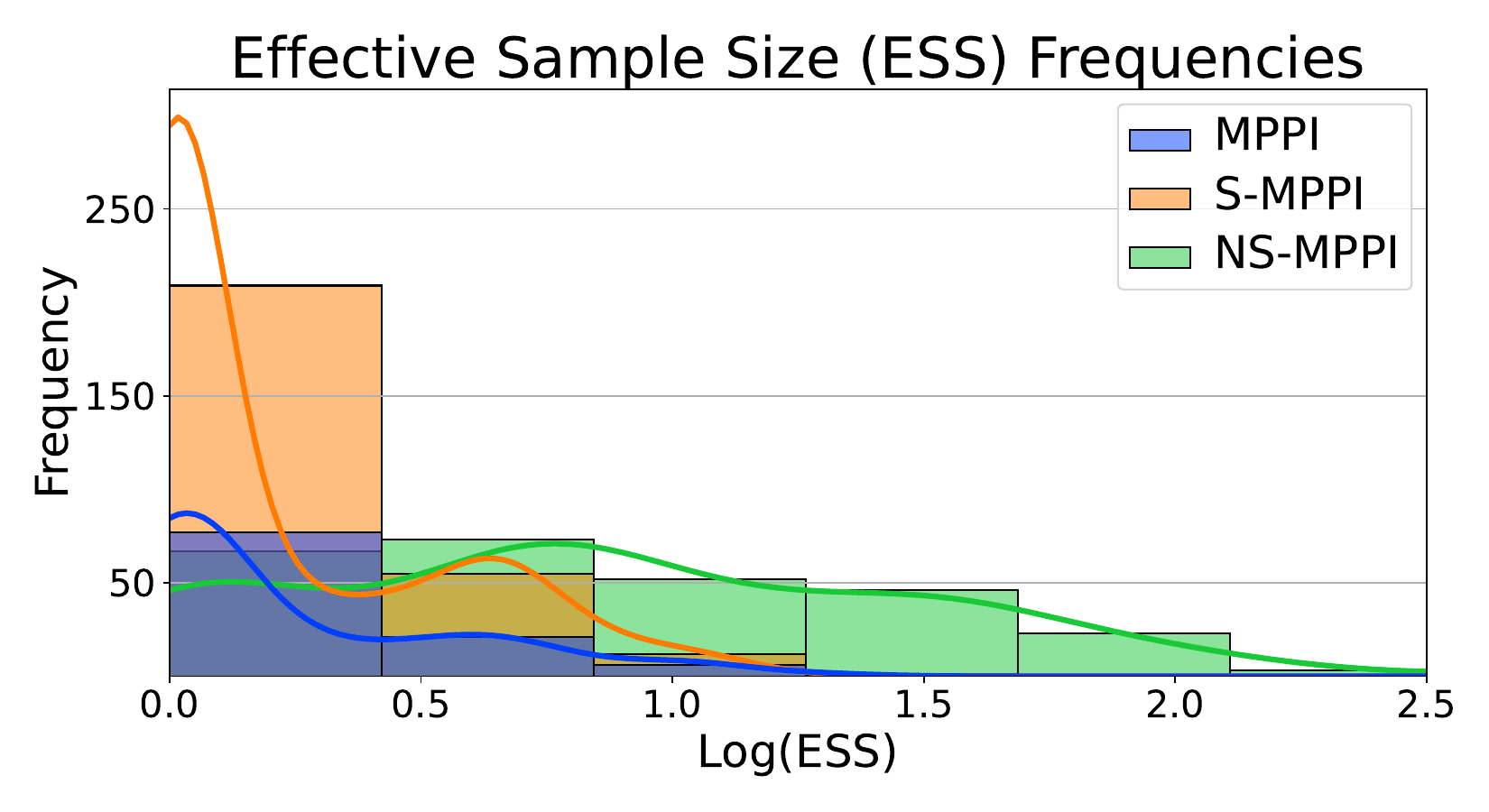}
    \vspace{-.6ex}
    \caption{\textbf{NS-MPPI has larger $\widehat{\BoldESS}$.} The proposed NS-MPPI achieves larger effective sample sizes, verifying that RBR enables more efficient use of samples and computations.}
    \label{fig:ess_text}
    \vspace{-1ex}
\end{figure}

\beforetextbfmore{}

\noindent\textbf{RBR achieves larger $\widehat{\BoldESS}$.} RBR theoretically improves the $\widehat{ESS}$, as mentioned in \Cref{sec:Efficient_Trajectory_Sampling}. 
We verified this claim empirically on \textsf{AutoRally} and plot the results in \Cref{fig:ess_text}.
While MPPI and S-MPPI exhibit values of $\widehat{ESS}$ concentrated near $1$,
the $\widehat{ESS}$ values for NS-MPPI are more uniformly distributed, indicating a more extensive utilization of the sampled trajectories and thus a better use of the available computational resources.

\begin{figure}[]
    \centering
    \begin{subfigure}[b]{0.47\linewidth}
        \centering
        \includegraphics[width=\textwidth]{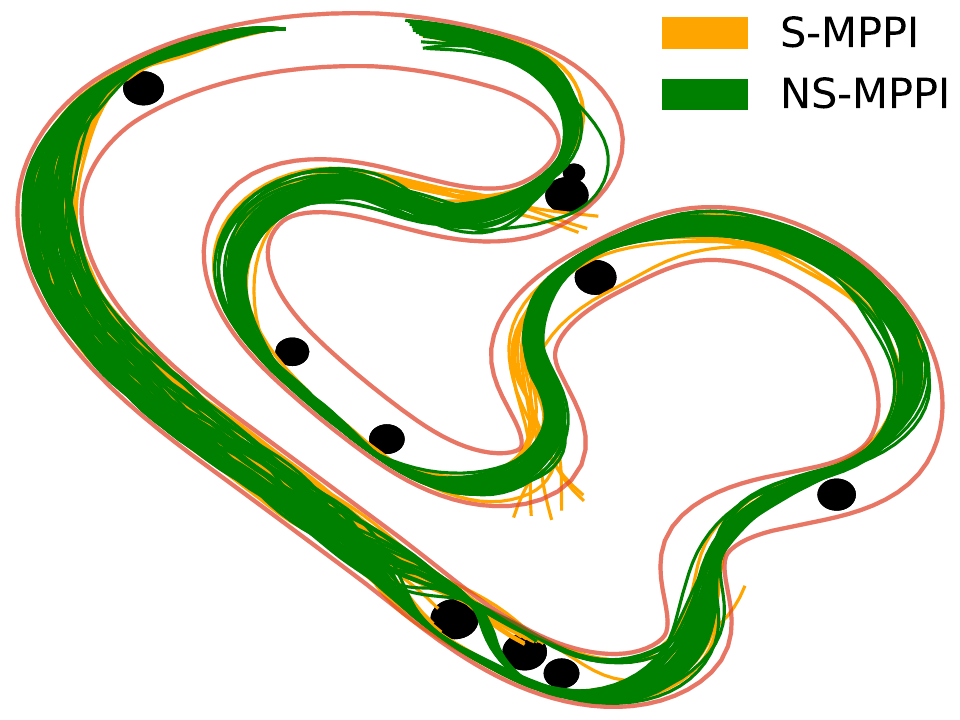}
        \caption{}
        \label{fig:ShieldMPPIvsNSMPPI}
    \end{subfigure}\hspace*{\fill}
    \begin{subfigure}[b]{0.51\linewidth}
        \centering
        \includegraphics[width=\textwidth]{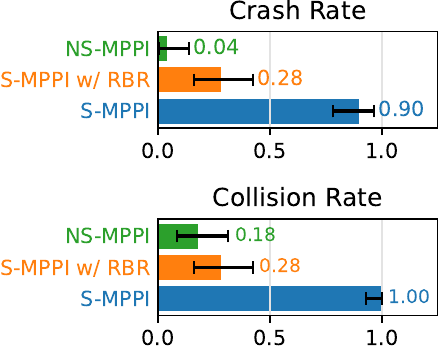}
        \caption{}
        \label{fig:crash_and_collision_rate_cluttered_env}
    \end{subfigure}
    
    \caption{\textbf{AutoRally with obstacles.} (a) While NS-MPPI consistently clears the entire track, S-MPPI often crashes into obstacles and walls. (b) NS-MPPI has $85\%$ fewer crashes compared to S-MPPI. Adding RBR alone to S-MPPI fixes $62\%$ of the crashes made by S-MPPI.}
    \label{fig:cluttered_autorally}
    \vspace{-1em}
\end{figure}

\beforetextbf{}

\noindent\textbf{RBR is especially beneficial in harder environments.}
To further stress-test RBR, we ran 50 trials with 100 sampled trajectories on a harder variant of \textsf{AutoRally} with obstacles of various sizes.
Note that the presence of obstacles close to the center of the track violates the assumption in \Cref{rem:convex_thing} regarding the safety of the mean $\vv$.
Comparing the trajectories from S-MPPI and NS-MPPI in \Cref{fig:ShieldMPPIvsNSMPPI}, S-MPPI is unable to avoid crashing into walls and obstacles since it relies on a heuristic DCBF.
On the other hand, NS-MPPI avoids crashes in almost all runs.
We also compare against S-MPPI w/ RBR and plot the crash and collision rates in \Cref{fig:crash_and_collision_rate_cluttered_env}, where we see that RBR constitutes a large fraction of the performance improvements, being responsible for $72\%$ of the crash rate reduction between S-MPPI and NS-MPPI.
Nevertheless, RBR alone is not sufficient, and the use of an approximate DPNCBF is still required to bring the crash rate to near $0$.

To demonstrate the enhancement in sampling efficiency achieved by the proposed sampling method, we provide visual representations of trajectory sampling distributions obtained from simulations in environments with obstacles in ~\Cref{fig:EfficientSamplingVsNormalSampling1}.
As demonstrated by ~\Cref{fig:EfficientSamplingVsNormalSampling1} (b), the RBR concentrates trajectory samples within a feasible narrow passage that satisfies the DCBF safety criteria. 
Conversely, the standard sampling approach, exemplified by \Cref{fig:EfficientSamplingVsNormalSampling1} (a),  produces a sparse distribution, misallocating samples to unsafe zones and thereby failing to sufficiently explore safe areas. 
This inadequacy leads to a collision with an obstacle.

\begin{table}[]
\caption{Performance and Timing Comparison}
\centering
\begin{tabular}{lSSS}
\toprule
 \textbf{Controller} & \textbf{Crash Rate}& \textbf{Collision Rate} & \textbf{Control Rate (\si{\hertz})}\\
\midrule
NS-MPPI      & 0.04 & 0.06 & 73.64 \\
S-MPPI      & 0.46 & 0.70 & 107.31 \\
MPPI      & 0.98 & 1.00 & 130.05 \\
CEM      & 1.00 & 1.00 & 88.17 \\
\bottomrule
\end{tabular}
\label{tab:CPU_table}
\end{table}

\beforetextbf{}

\noindent\textbf{Enhanced sampling efficiency enables safe real-time planning on a CPU.}
We compare the computation times of different VI-MPC controllers on \textsf{AutoRally} with a $\qty{12}{\meter\per\second}$ target speed, horizon of $K=15$ and small sample size of $N=30$ trajectories per optimization iteration (\Cref{tab:CPU_table}).
NS-MPPI achieves the lowest crash and collision rates, albeit with a slightly lower control rate of $>\qty{70}{\hertz}$ that suffices for most robotic tasks.
S-MPPI suffers from over $11$ times larger crash and collision rates, while both MPPI and CEM controllers exhibit crash and collision rates close to $1$.

\begin{figure}[]
    \centering
    \includegraphics[width=\linewidth]{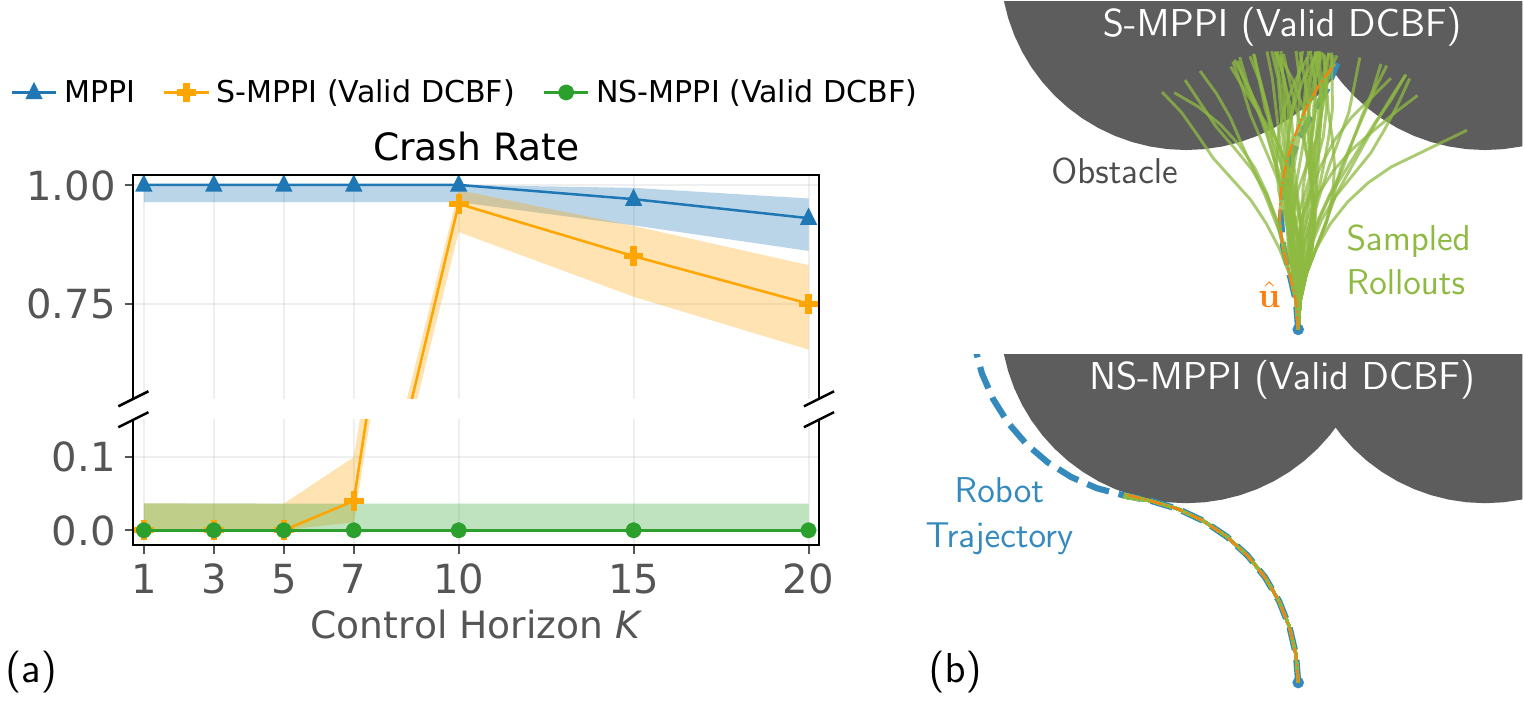}
    \caption{\textbf{RBR improves performance even with \textit{valid} DCBF.} (a)
    As the estimator variance grows exponentially with K, S-MPPI crashes at larger K \textit{despite} using a valid DCBF.
    Using RBR (NS-MPPI) mitigates this, maintaining safety across all horizons.
    (b) We visualize this exponential growth in variance by comparing the rollouts of the two methods at $K=10$. S-MPPI is unable to sample a control sequence that turns left for all $K=10$ steps with only $N=50$ samples.
    }
    \label{fig:rbr_valid_dcbf}
\end{figure}

\beforetextbf{}

\noindent\textbf{RBR improves performance even with a \textit{valid} DCBF.}
We further investigate whether RBR is beneficial when we have a \textit{valid} DCBF by performing experiments on \textsf{Dubins}, a Dubins car that travels with fixed velocity, using only $N=50$ samples (\Cref{fig:rbr_valid_dcbf}).
Here, NS-MPPI (Valid DCBF) uses RBR while S-MPPI (Valid DCBF) does not. Both methods use a valid DCBF.
At $K=1$, RBR has no effect since no resampling occurs, but both methods avoid crashes.
However, since the estimator variance grows exponentially in the horizon length $K$ (\Cref{sec:Efficient_Trajectory_Sampling}), S-MPPI crashes at $K\geq7$, where $N=50$ samples is not enough to sample a trajectory that satisfies the DCBF constraints at all timesteps. Using RBR mitigates this exponential growth, and NS-MPPI remains safe across all time horizons. See \Cref{app:ValidDCBF} for more details and additional plots.

\subsection{Comparing DPNCBF against grid-based reachability.}
Finally, we perform case studies to investigate the learned DPNCBF. In particular, we question whether the value function obtained via continuous-time techniques can be used in place of the proposed DPNCBF.

\beforetextbfok{}

\begin{figure}[]
    \centering
    \includegraphics[width=.7\linewidth]{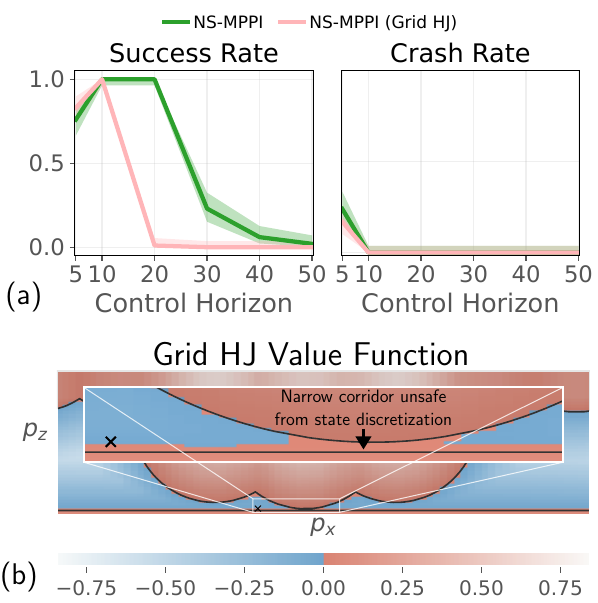}

    \caption{\textbf{Value functions from grid-based solvers suffer in multi-fidelity problems.}
    (a) Using a value function from a grid-based HJ reachability solver results in much lower success rates at longer control horizons compared to the learned DPNCBF.
    (b) Due to state discretization, the narrow corridor is marked unsafe, causing the drone to be overly conservative and not cross the corridor to remain safe. 
    }
    \label{fig:drone_grid}
    \vspace{-0.7em}
\end{figure}

\noindent\textbf{Grid-based reachability suffers with multi-fidelity problems.}
We compare against a grid-based reachability method solved using methods from \textit{continuous-time} HJ reachability \cite{mitchell2008flexible} (denoted NS-MPPI (Grid)) on \textsf{Drone} in \Cref{fig:drone_grid}.
The height of the narrow corridor is much smaller than the length of the overall space. Consequently, given the $6$-dimensional state-space of \textsf{Drone}, the grid size used by the grid-based solver is too coarse to capture the safe region in the narrow corridor accurately and marks the narrow corridor as unsafe, blocking traversal.
This causes NS-MPPI (Grid HJ) to have a much lower success rate than NS-MPPI at longer control horizons.

\beforetextbfok{}

\begin{figure*}
    \centering
    \includegraphics[width=\linewidth]{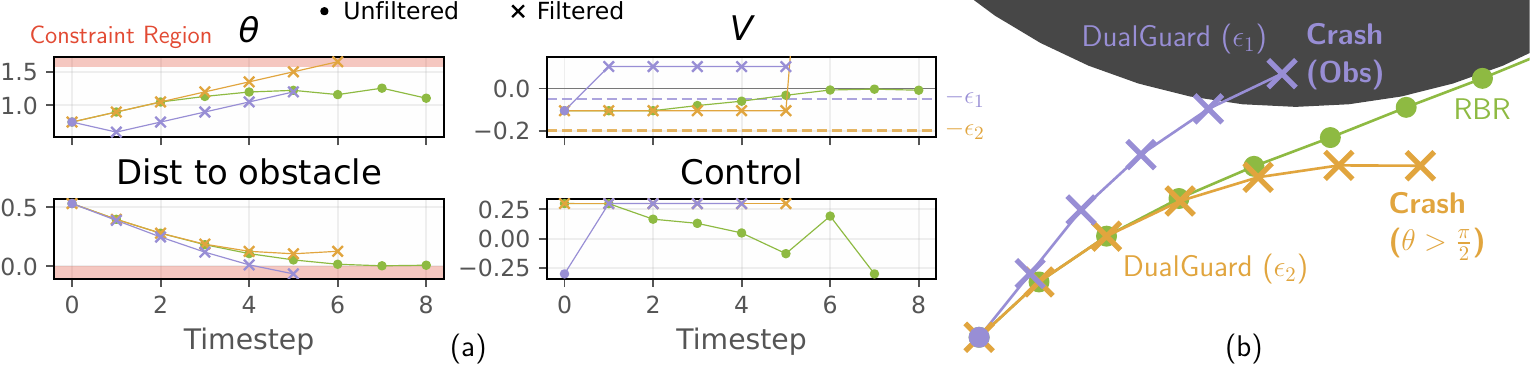}
    \caption{\textbf{Optimal safe controls from continuous-time reachability can cause crashes in discrete-time.}
    The optimal safe control from continuous-time is different from the optimal control in discrete-time.
    Consequently, using the value function from continuous-time HJ reachability as a safety filter during the rollout process (i.e., concurrent work DualGuard-MPPI \cite{borquez2025dualguard}) leads to crashes, even if a margin $\epsilon$ is added to try to account for this gap by being more conservative.
    }
    \label{fig:dubins_dualguard}
\end{figure*}

\noindent\textbf{Optimal safe controls from continuous-time reachability can cause crashes in discrete-time.}
Finally, we compare against a concurrent work DualGuard-MPPI \cite{borquez2025dualguard} that uses the value function solved using continuous-time HJ reachability as a safety filter during the rollout process of MPPI on \textsf{Dubins} (\Cref{fig:dubins_dualguard}), where we add an additional box-constraint $\abs{\theta} \leq \pi/2$ on the heading.
For this experiment, to rule out crashed caused by errors in the learned DPNCBF, we wrote a \textit{discrete-time} grid-based solver and used the corresponding value function as the DCBF.
Different from our method, DualGuard uses the value function as a safety filter during the rollout, replacing the original control with the \textit{continuous-time} optimal safe control when the \textit{current} value function is unsafe $V(x) > -\epsilon$, with a static margin $\epsilon$ added to try to account for the gap between continuous-time and discrete-time.
However, \Cref{fig:dubins_dualguard} shows that neither small nor large value of $\epsilon$ can bridge this gap, and DualGuard-MPPI violates safety constraints in both cases.
When the margin is small ($\epsilon = \epsilon_1$), the safety filter acts too late to prevent collision with the obstacle. When the margin is larger ($\epsilon = \epsilon_2$), the safety filter acts early and avoids the obstacle, but overshoots and violates the $\theta$ constraint instead due to the zero-order hold used in discrete-time.

\section{Hardware Experiments}

Finally, we deployed our method and other sampling-based MPC baselines on hardware using the AutoRally experimental platform \cite{goldfain2019autorally}.
Testing was conducted on a dirt track outlined by drainage pipes, as shown in \Cref{fig:AutoRallyTrackSetUp}.
All computations use the onboard Intel i7-6700 CPU and Nvidia GTX-750ti GPU \cite{goldfain2019autorally}.
See \Cref{app:hw_details} for more details on the hardware experiments.

\begin{figure}[]
    \centering
    \centerline{\includegraphics[width=9cm,height=4cm]{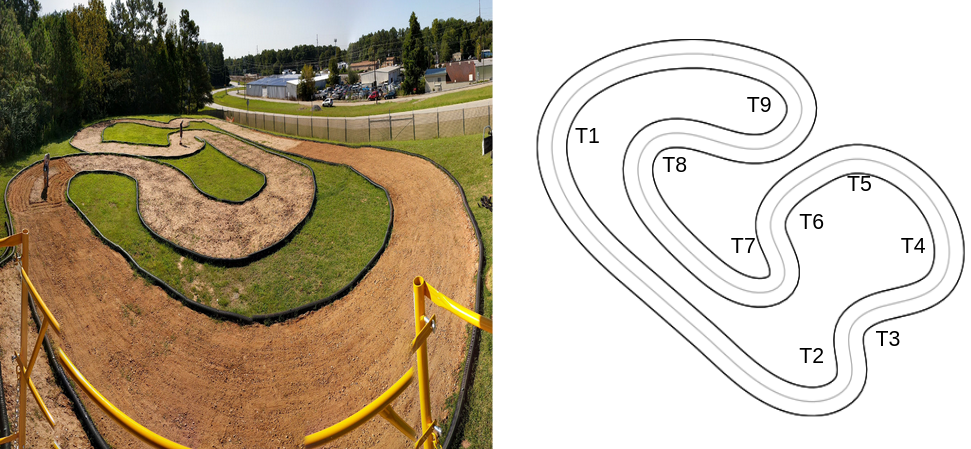}}
    \caption{\textbf{AutoRally Track Setup.} 
    The site is equipped with a spacious carport, a 3-meters-tall observation tower, and a storage shed \cite{AutorallyHardware}.
    Each turn of the track is numbered as Ti, with T1 is the first turn and T9 is the last.
    }
    \label{fig:AutoRallyTrackSetUp}
\end{figure}

\subsection{Robustness Against Adversarial Costs}

Robots depend on sensors such as cameras to gather information about their environment, and use the collected data to construct suitable cost functions. 
Deep Neural Networks (DNNs) are frequently utilized for object detection, classification, and semantic segmentation \cite{zhao2017survey}. 
However, DNNs are susceptible to adversarial attacks \cite{PhysicalAdversarialAttack}, which can mislead the algorithms into generating incorrect outputs. 
These attacks often involve altering the appearance of the object of interest directly \cite{li2019adversarial}, eventually resulting in erroneous cost functions.
In this hardware experiment, we tested the robustness of the proposed NS-MPPI to these erroneous or misspecified cost functions by using an \textit{adversarial} cost function on the actual AutoRally hardware platform.
Namely, instead of penalizing infeasible system states, we \textit{reward} unsafe states by assigning negative costs, effectively incentivizing collisions.
We compared against the S-MPPI and MPPI baselines.

The results are depicted in \Cref{fig:adversarial_cost_test}.
As MPPI only uses cost information, it rapidly steers the AutoRally vehicle toward the closest boundary, resulting in a crash.
S-MPPI crashes the vehicle at the first turn (T1) as well, but it covers a greater distance than MPPI.
Only NS-MPPI achieves a collision-free completion of 10 laps.
Nonetheless, there are noticeable oscillatory movements of the car, indicating susceptibility to attractive costs driving it toward boundary collisions.

\begin{figure}[]
    \centering
    \centerline{\includegraphics[width=\linewidth]{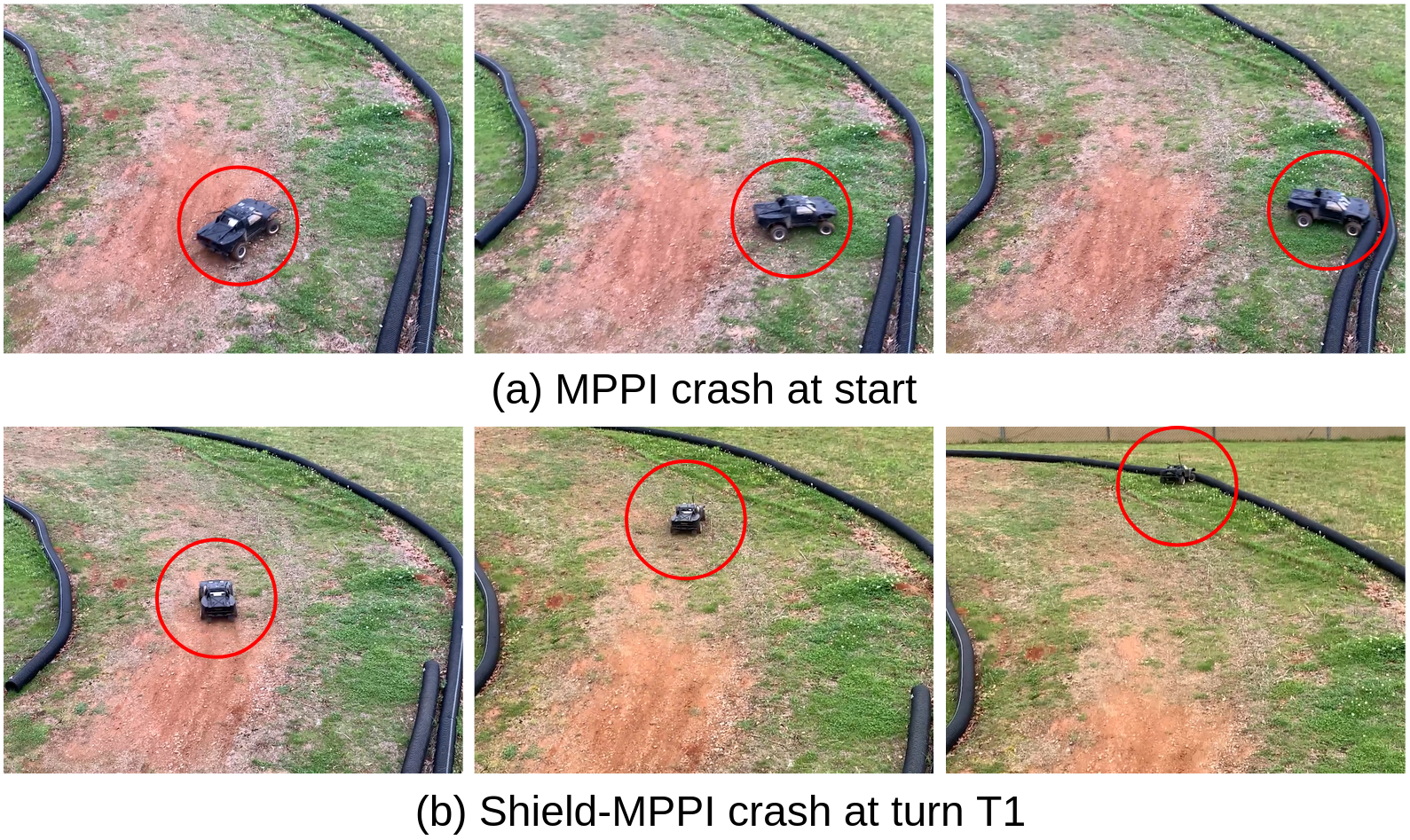}}
    \caption{\textbf{Adversarial cost.} With a cost function that rewards crashes, both MPPI and S-MPPI crash. Only NS-MPPI completes 10 laps collision-free.}
    \label{fig:adversarial_cost_test}
\end{figure}

\begin{figure}[]
    \centering
    \centerline{\includegraphics[width=\linewidth]{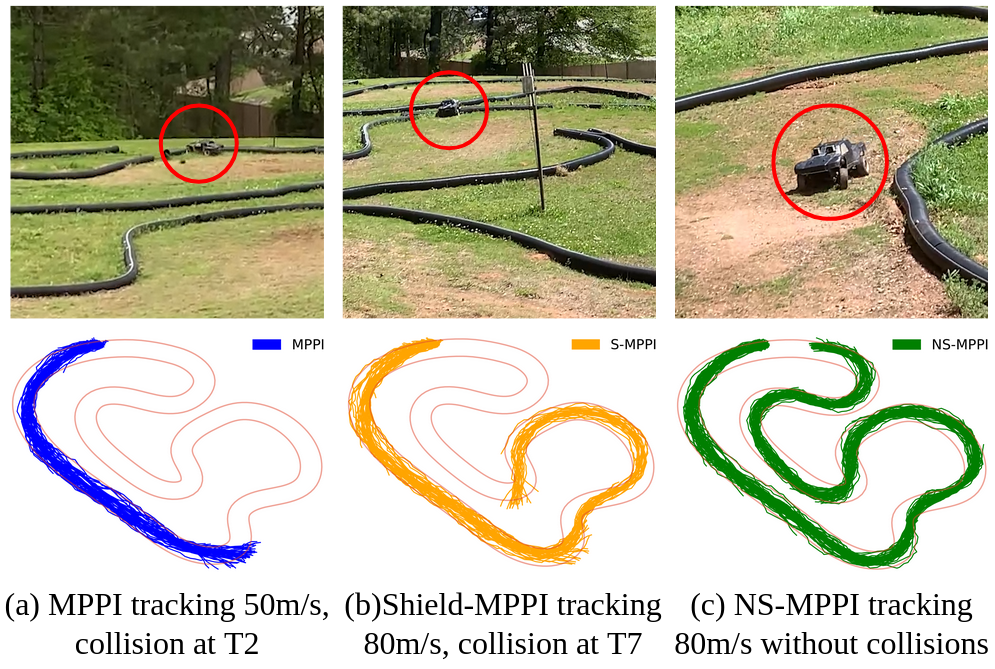}}
    \caption{\textbf{Safety under unsafe user inputs.} Using an erroneously large velocity target, both MPPI and S-MPPI crash.
    The visualization on the bottom row visualizes simulation trajectories that lead to crashes in the same turn.
    }
    \label{fig:safety_under_unsafe_user_inputs}
\end{figure}

\subsection{Safety Under Unsafe User Input}
Exceeding speed limits is a significant contributor to road accidents \cite{NHTSA_speeding} due to noncompliance by many drivers. The proposed Shield VIMPC control framework can smartly manage speeds supervised by the NCBF to achieve safety, even given unsafe user inputs. 
To this end, we tested MPPI, S-MPPI, and the proposed NS-MPPI on the AutoRally hardware using large, unsafe target velocities to examine their safety performance. 
The results are shown in \Cref{fig:safety_under_unsafe_user_inputs}. 
Designating high target velocities encourages maximizing speed, thus disrupting the equilibrium in the controllers' cost structure by diminishing the emphasis on cost penalties associated with obstacles. As a result, MPPI causes the vehicle to deviate from the path, resulting in a crash at the second turn (T2), while S-MPPI impacts at the seventh turn (T7). 
In contrast, NS-MPPI ensures safety, preventing any accidents.

\section{Limitations}

One limitation of our approach is that the DPNCBF is only an approximation of a DCBF. 
As such, the typical safety guarantees of DCBFs may not hold if the function is not a true DCBF.
Moreover, as noted in \Cref{rem:convex_thing}, the use of variational inference means that $\vv^*$ may not satisfy the state constraints despite the optimal distribution $p(\vu \mid o=1)$ having zero density at states that violate the state constraints. 
While theoretically we were only able to show that $\vv^*$ stays safe under the assumption that $\mathcal{U}^K_{\text{safe}}$ is convex, empirical results in \Cref{fig:cluttered_autorally} show that NS-MPPI can perform well despite the fact that $\mathcal{U}^K_{\text{safe}}$ is typically not convex.

\section{Conclusion}

In this study, we have adapted the policy neural control barrier function (PNCBF) to a discrete-time setting and utilized the resulting discrete-time policy neural control barrier function (DPNCBF) to supervise the proposed resampling-based rollout (RBR) method.
This novel method enhances sampling efficiency and safety for all variational inference model predictive controllers (VIMPCs).
Leveraging RBR, we developed the neural shield VIMPC (NS-VIMPC) control framework and demonstrated its benefits for safe planning through the novel neural shield model predictive path integral (NS-MPPI) controller. We conducted tests of the NS-MPPI, benchmarking its performance against state-of-the-art sampling-based MPC controllers using both an autonomous vehicle and a drone. 
Our simulations and experimental data indicate that the NS-MPPI outperforms existing VIMPC methods in terms of safety and sampling efficiency. 
The improved sampling efficiency allows NS-MPPI to reach similar performance levels as other VIMPC methods, while using significantly fewer sampled trajectories, facilitating real-time control on a CPU rather than necessitating costly GPU resources.

We have used the novel RBR method to concentrate sampled trajectories on safe regions. While this approach can significantly improve sampling efficiency in terms of safety, it does not improve other aspects of performance. 
To further enhance performance, we may, for instance, integrate a Control Lyapunov Function \cite{dawson2023safe} or a Control Lyapunov Barrier Function \cite{CLBF} with the novel RBR approach to sample safe and higher-performance trajectories.

\section*{Acknowledgments}
This work was funded by the National Science Foundation (NSF) under CAREER award CCF-2238030 and award CNS-2219755, and Office of Naval Research (ONR) under award N00014-23-1-2353.

\bibliographystyle{plainnat}
\bibliography{bib/references}

\newpage
\onecolumn
\begin{appendices}
\numberwithin{equation}{section}
\renewcommand{\theequation}{\thesection.\arabic{equation}}
\renewcommand{\thesubsection}{\thesection\arabic{subsection}}
\renewcommand{\thesubsectiondis}{\thesubsection}

\crefalias{section}{appendix}

\section{Details on the Simulation Experiments}\label{app:sim_details} %

\subsection{AutoRally Environment} \label{app:sim_details:autorally}
The numerical simulations use the AutoRally platform, a 1/5-scale autonomous vehicle system capable of executing high-speed, limit-handling maneuvers in complicated environments \cite{goldfain2019autorally}.
Figure\ref{fig:AutorallyChassisAndComputeBox} shows the AutoRally hardware.
In this study, we define a crash as an event when the vehicle makes contact with the track boundary, constructed from soft drainage pipes, resulting in the inability to continue driving. We define a collision as an event when the vehicle contacts the track boundary but can continue driving.

\begin{figure}[htb]
    \centering
    \centerline{\includegraphics[scale=1.3]{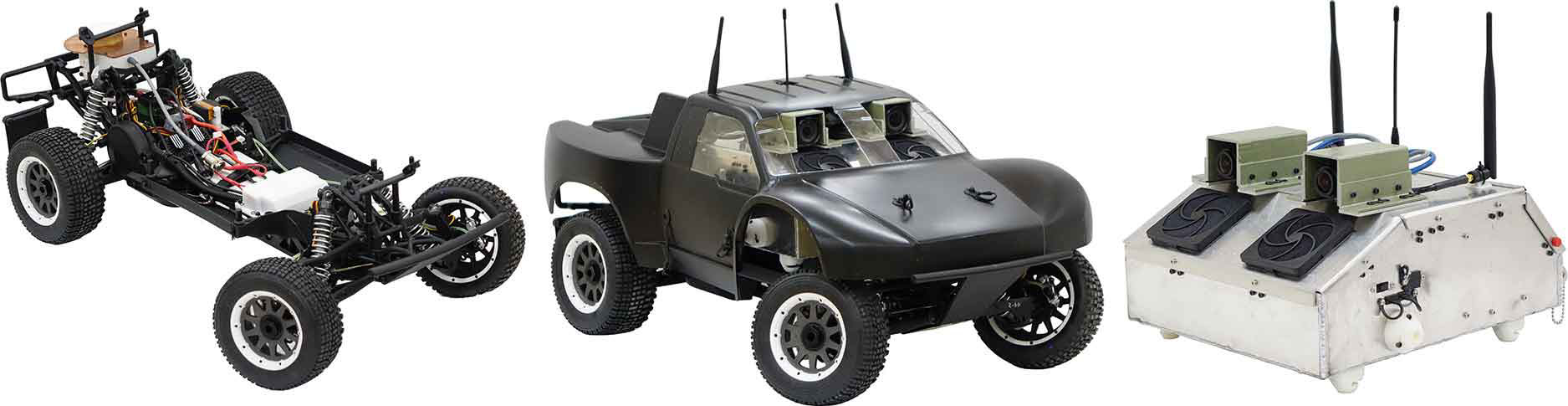}}
    \caption{AutoRally chassis and compute box \cite{AutorallyHardware}.}
    \label{fig:AutorallyChassisAndComputeBox}
\end{figure}

\subsection{AutoRally Dynamics Modelling}\label{AppendixA} %
We model the AutoRally using a rear-wheel drive, single-track bicycle model \cite{CSSMPC}.
The vehicle system is represented in curvilinear coordinates using the centerline of the track as the reference curve. Using a curvilinear coordinate system provides a more intuitive interpretation of the vehicle's position and heading relative to the track, as compared to Cartesian coordinates.

To this end, we consider a nonlinear, continuous-time system,
\begin{align}\label{eqn:VehicleDynamicsContinuousSystem}
    \dot{x} = F(x, u),
\end{align}
where the system state of the AutoRally is given by,
\begin{align}
	x &= \begin{bmatrix}
		v_x, v_y, \dot{\psi}, \omega_F, \omega_R, e_\psi,  e_y, s
	\end{bmatrix}^\intercal,
	\label{eqn:autorally_state_space}
\end{align}
 and where $v_x$ is the longitudinal velocity, $v_y$ is the lateral velocity,  and $\dot{\psi}$ is the yaw rate. The front and rear wheel angular velocities are denoted by $\omega_F$ and $\omega_R$.
The yaw error and the lateral distance error from the centerline of the track are denoted by $e_\psi$ and $e_y$, respectively (see \Cref{fig:wheel_velocities}). 
The variable $s$ is the curvilinear position along the track centerline.
The control input to the system \eqref{eqn:VehicleDynamicsContinuousSystem} is,
\begin{align}
    u &= \begin{bmatrix}
		\delta, T
	\end{bmatrix}^\intercal,
\end{align}
where $\delta$ is the steering angle and $T$ denotes the values for throttle input (if positive) or braking input (if negative). The dynamics of a single-track dynamic bicycle model \cite{CSSMPC} used to model \eqref{eqn:VehicleDynamicsContinuousSystem} are given by,
\begin{subequations}\label{eqn:VehicleDynamics}
\begin{align} 
\dot{v}_{x} &= \frac{f_{F x}\cos\delta - f_{F y}\sin\delta + f_{R x}}{m} + v_{y} \dot{\psi}, \\ 
\dot{v}_{y}&= \frac{f_{F x}\sin\delta + f_{F y}\cos\delta + f_{R y}}{m} - v_x \dot{\psi}, \\ 
\ddot{\psi} &=\frac{\left(f_{F y}\cos\delta + f_{F x}\sin\delta\right) \ell_{F} - f_{R y}\ell_{R}}{I_z},\\
\dot{\omega}_F &= -\frac{r_F}{I_{\omega F}}f_{Fx},\\
\dot{\omega}_R &= \Theta(\omega_R, T), \label{eqn:rear_wheel_acceleration}\\
\dot{e}_\psi &= \dot{\psi} - \frac{v_{x} \cos{e_{\psi}} - v_{y} \sin{e_{\psi}}}{1 - \rho(s) e_{y}} \rho(s), \\
\dot{e}_y &= v_{x} \sin{e_{\psi}} + v_{y} \cos{e_{\psi}},\\
\dot{s} &= \frac{v_{x} \cos{e_{\psi}} - v_{y} \sin{e_{\psi}}}{1 - \rho(s) e_{y}},
\end{align}
\end{subequations}
where $m$ and $I_z$ are the vehicle's mass and moment of inertia about the vertical axis, respectively. The radius of the front wheel is $r_F$, and the moment of inertia of the front wheel about the front axle is $I_{\omega F}$. The curvature of the track centerline at position $s$ is $\rho(s)$.
\begin{figure}[!h]
    \centering
    \includegraphics[scale=0.18]{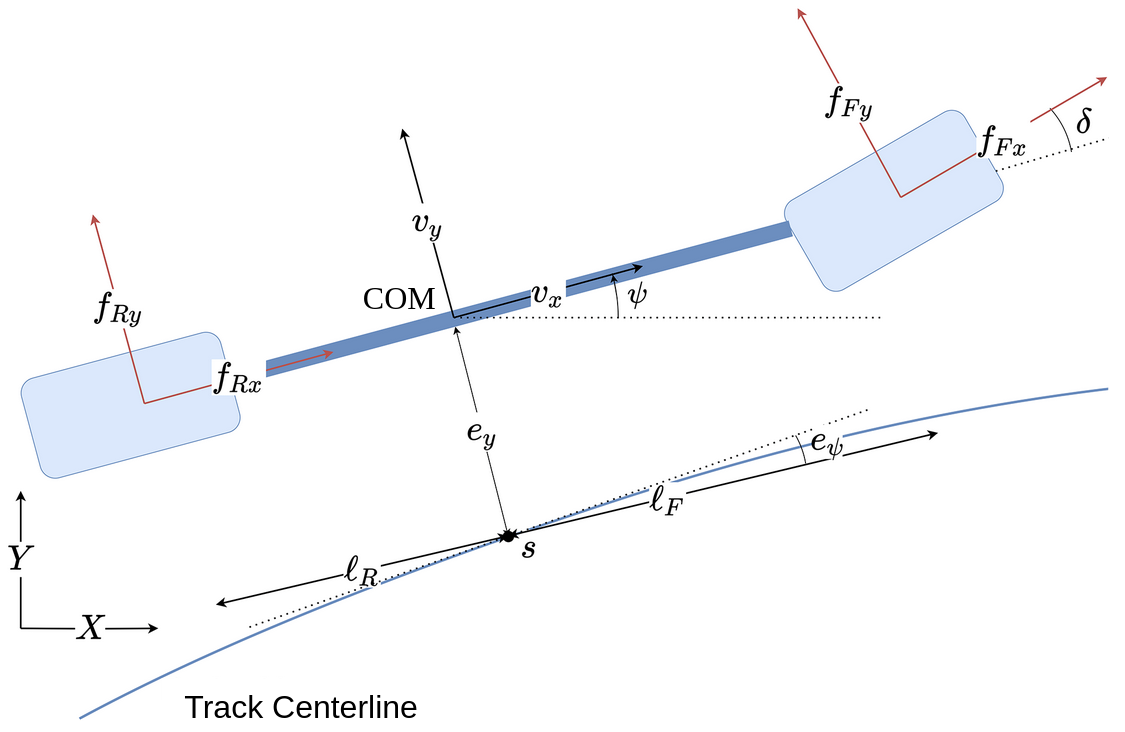}
    \caption{Schematic of the dynamic bicycle model.}
    \label{fig:wheel_velocities}
\end{figure}
The front and rear tire frictional forces are denoted by $f_{Fx}, f_{Fy}, f_{Rx}, f_{Ry}$, where the subscripts $F, R$ indicate front and rear tires and $x, y$ indicate longitudinal and lateral directions. These frictional forces are computed using the ellipse model in \cite{velenis2010steady},
\begin{subequations}\label{egs:fric_from_normals}
\begin{align}
    f_{ix} &= f_{iz}\mu_{ix}, \\
    f_{iy} &= f_{iz}\mu_{iy} + \Phi_i(\delta, \alpha_i), \label{f_iy}
\end{align}
\end{subequations}
where $i = F, R$ indicates front or rear wheels, $f_{iz}$ is the normal force acting on the tire. The tire friction coefficients $\mu_{ix}, \mu_{iy}$ are dependent on wheel slip angles and the intrinsic tire parameters. 

Inspired by \cite{acosta2018tire}, we use a neural network to model the residual error $\Phi_i(\delta, \alpha_i)$, where $\delta$ is steering angle, $\alpha_i = \arctan(v_{iy}/v_{ix})$ is wheel slip angle. $\Phi_i$ is trained using experimental data collected from the AutoRally. The rear wheel angular acceleration $\dot{\omega}_R$ is given by \eqref{eqn:rear_wheel_acceleration}. Modelling $\dot{\omega}_R$ is challenging, because it is determined by a variety of mechanical and physical conditions, including nonlinear tire behavior, dynamic load transfer, road surface irregularities and the complex interaction between these components. To this end, we use a data-driven approach and utilize a neural network $\Theta(\cdot, \cdot)$ trained on experimental data to model the rear wheel angular acceleration $\dot{\omega}_R$.

The nonlinear, continuous-time system \eqref{eqn:VehicleDynamicsContinuousSystem} is discretized and converted to a discrete-time system using Euler integration,
\begin{align}\label{eqn:VehicleDynamicsDiscrete}
    x_{k+1} = f(x_k, u_k) = x_k + F(x_k, u_k)\Delta t,
\end{align}
where $\Delta t = t_{k+1} - t_k$ is the discretization time interval.
We employ an Unscented Kalman Filter (UKF) on the discrete-time system \eqref{eqn:VehicleDynamicsDiscrete} to identify the vehicle and tire parameters using the approach in \cite{changxi}. The resulting parameters of the vehicle model include the vehicle mass $m = 22~\mathrm{kg}$, moment of inertia $I_z = 1.1 ~\mathrm{kg\cdot m^2}$, $l_F = 0.34~\mathrm{m}$, $l_R = 0.23~\mathrm{m}$, $I_{\omega F} = 0.10 ~\mathrm{kg\cdot m^2}$ and front wheel radius $r_F = 0.095~\mathrm{m}$. The parameters in Pacejka's Magic Formula used by the tire friction elipse model \cite{velenis2010steady} are computed as $\tire_\rB = 4.1$, $\tire_\rC=0.95$, $\tire_\rD=1.1$.

\subsection{Drone Dynamics Modelling}\label{drone_dynamics_modelling}
We model the two-dimensional drone dynamics using the following equations,
\begin{subequations}\label{eqn:DroneDynamics}
\begin{align} 
\dot{v}_{x} &= -\frac{F}{m}\sin{\theta}, \\ 
\dot{v}_{z}&= \frac{F}{m}\cos{\theta} - g,\\
\ddot{\theta} &= \frac{\tau}{I},\\
\end{align}
\end{subequations}
where the total thrust $F$ and the torque $\tau$ are given by,
\begin{align}
    F = F_1 + F_2, \quad \tau = \ell(F_2 - F_1),
\end{align}
where $\ell$ is the length from the center of the drone to the center of each rotor, and $F_1$, $F_2$ are thrust forces generated by the left and right rotors, respectively. The thrust forces are computed considering the ground effect \cite{ground_effect_book}, given by,
\begin{align}
    F_1 = F_{in1}/\bigl(1-\rho(\frac{r}{4z_{r1}})\bigr), \\
    F_2 = F_{in2}/\bigl(1-\rho(\frac{r}{4z_{r2}})\bigr),
\end{align}
where $F_{in1}$, $F_{in2}$ are the input thrust commands for the left and right rotors, and $\rho$ is an intrinsic property of the drone. The radius of the rotors is $r$ and the heights from the rotors to the ground are $z_{r1}$ and $z_{r2}$, which are given by,
\begin{align}
    z_{r1} = z - \ell\sin{\theta},\\
    z_{r2} = z + \ell\sin{\theta},
\end{align}
where $\theta$ is the pitch angle.

\section{Proofs for Variational Inference MPC} 
\subsection{Derivations of the Variational Inference Updates} \label{app:VI Gauss}
Let $q_\vv$ be the density function of a Gaussian distribution with mean $\vv$ and covariance $\Sigma$, i.e.,
\begin{equation}
    q_\vv(\vu) = Z_q^{-1} \exp\left(-\frac{1}{2}(\vu-\vv)^\top\Sigma^{-1}(\vu-\vv)\right),
\end{equation}
where the normalization constant $Z_q$ is independent of the mean $\vv$.
We wish to solve the variational inference problem
\begin{equation}
    \min_{\vv} \quad \KL{ p(\vu \mid o=1) }{ q_\vv(\vu) }.
\end{equation}
Expanding and ignoring terms unrelated to $\vv$, we get
\begin{align}
    \KL{ p(\vu \mid o=1) }{ q_\vv(\vu) }
    &= \int p(\vu \mid o=1) \log\left(\frac{p(\vu \mid o=1)}{q_\vv(\vu)}\right) \diff\vu \\
    &= \int p(\vu \mid o=1) \log\left(p(\vu \mid o=1)\right) \diff\vu -\int p(\vu \mid o=1) \log(q_\vv(\vu)) \diff\vu\\
    &= -\int p(\vu \mid o=1) \log(q_\vv(\vu)) \diff\vu + O(1) \\
    &=  \int p(\vu \mid o=1) \left( \frac{1}{2} (\vu - \vv)\T \Sigma^{-1} (\vu - \vv) + \log Z_q \right) \diff\vu + O(1) \\
    &= \E_{p(\vu \mid o=1)}\left[ \frac{1}{2} (\vu - \vv)\T \Sigma^{-1} (\vu - \vv) + \log Z_q \right] + O(1) \\
    &= \E_{p(\vu \mid o=1)}\left[ \frac{1}{2} (\vu - \vv)\T \Sigma^{-1} (\vu - \vv) \right] + O(1).
\end{align}
Taking the derivative with respect to $\vv$ and setting to zero, we get
\begin{equation}
    \E_{p(\vu \mid o=1)}\left[ \Sigma^{-1} (\vu - \vv^*) \right] = 0.
\end{equation}
Rearranging, yields
\begin{equation}
    \vv^* = \E_{p(\vu \mid o=1)}\left[ \vu \right].
\end{equation}

\subsection{Derivations of the Self Normalized Importance Sampling Estimator} \label{app:VI SNIS}
We next derive the self-normalized importance sampling (SNIS) estimator $\hat{\vv}$ in \eqref{eq:vi:v_opt} of $\vv^*$ in \eqref{eq:vi:omega_def}.
Note that, since $Z=\int \exp(-J(\vu)) p_0(\vu) \diff{\vu}$ is unknown, we cannot compute the weights $\omega(\vu)$ in \eqref{eq:vi:omega_def} directly.
Hence, the regular importance sampled Monte Carlo estimator
\begin{equation} \label{eq:vi_deriv:is_estimator}
    \hat{\vv} = \frac{1}{N} \sum_{i=1}^N \omega(\vu^i) \vu^i,
\end{equation}
is not computable.
Instead, we can use the SNIS estimator \cite{owen2013monte}, which can be derived as follows.
First, note that $Z$ can be written as the expectation of $\exp(-J(\vu))$, i.e.,
\begin{equation}
    Z = \E_{p_0(\vu)}\left[ \exp(-J(\vu)) \right].
\end{equation}
Hence, one idea is to reuse the samples $\vu^i \sim r$ to compute an estimate $\hat{Z}$ of $Z$, that is,
\begin{equation}
    \hat{Z} = \frac{1}{N} \sum_{i=1}^N \exp(-J(\vu^i)) p_0(\vu^i) / r(\vu^i),
\end{equation}
where $\hat{Z}$ is a ``normal'' importance sampled Monte Carlo estimator of $Z$.
If we use the estimate $\hat{Z}$ to compute the weights $\hat{\omega}(\vu)$ in the normal importance sampled Monte Carlo estimator $\hat{\vv}$ \eqref{eq:vi_deriv:is_estimator}, we obtain the SNIS estimator as in \eqref{eq:vi:v_opt} and \eqref{eq:vi:omega_def}
\begin{align}
    \hat{\vv}
    &= \frac{1}{N} \sum_{i=1}^N \hat{\omega}(\vu^i) \vu^i, \\
    &= \frac{1}{N} \sum_{i=1}^N \left( \frac{1}{ \hat{Z} } \frac{ \exp(-J(\vu^i))p_0(\vu^i) }{ r(\vu^i) }  \right) \vu^i, \\
    &= \frac{1}{N} \sum_{i=1}^N \left( \frac{1}{ \frac{1}{N} \sum_{j=1}^N \exp(-J(\vu^j)) p_0(\vu^j) / r(\vu^j) } \frac{ \exp(-J(\vu^i))p_0(\vu^i) }{ r(\vu^i) }  \right) \vu^i, \\
    &= \frac{1}{N} \sum_{i=1}^N \left( \frac{1}{ \frac{1}{N} \sum_{j=1}^N Z^{-1} \exp(-J(\vu^j)) p_0(\vu^j) / r(\vu^j) } \frac{ Z^{-1} \exp(-J(\vu^i))p_0(\vu^i) }{ r(\vu^i) }  \right) \vu^i, \\
    &= \frac{1}{N} \sum_{i=1}^N \left( \frac{1}{ \frac{1}{N} \sum_{j=1}^N Z^{-1} \exp(-J(\vu^j)) p_0(\vu^j) / r(\vu^j) } \omega(\vu^i)  \right) \vu^i, \\
    &= \frac{1}{N} \sum_{i=1}^N \left( \frac{\omega(\vu^i)}{ \frac{1}{N} \sum_{j=1}^N \omega(\vu^j) }  \right) \vu^i, \\
    &= \sum_{i=1}^N \left( \frac{\omega(\vu^i)}{ \sum_{j=1}^N \omega(\vu^j) }  \right) \vu^i, \\
    &= \sum_{i=1}^N \tilde{\omega}^i \vu^i,
\end{align}
where,
\begin{align}
    \tilde{\omega}^i = \frac{\omega(\vu^i)}{\sum^N_{j=1}\omega(\vu^j)}.
\end{align}
Unfortunately, the use of the estimated $\hat{Z}$ in the weights $\omega(\vu)$ causes the SNIS estimator to be biased \cite{owen2013monte}.

\subsection{MPPI as a Special Case of the Variational Inference Update} \label{app:MPPI as VI}
\begin{theorem}
    When $p_0(\vu) = q_{\bf 0}(\vu) := q_{\bf v = 0}(\vu)$ and $r(\vu)=q_{\bar{\vv}}(\vu)$ for some previous estimate of the optimal control sequence $\bar{\vv}$, the equation for the weights $\omega(\vu)$ in the variational inference MPC update \eqref{eq:vi:omega_def} reduces to that of MPPI.
\end{theorem}

\begin{proof}
    Let $q_\vv$ be the density function for a Gaussian distribution with mean $\vv$ and covariance $\Sigma$, that is,
    \begin{equation}
        q_\vv(\vu) = Z_q^{-1} \exp\left(-\frac{1}{2}(\vu-\vv)^\top\Sigma^{-1}(\vu-\vv)\right),
    \end{equation}
    By choosing $p_0(\vu) = q_{\bf 0}(\vu)$ and $r(\vu)=q_{\bar{\vv}}(\vu)$, we have that
    \begin{align}
        \frac{p_0(\vu)}{r(\vu)}
        &= \frac{ Z_q^{-1} \exp\left(-\frac{1}{2} \vu^\top\Sigma^{-1}\vu \right) }
        { Z_q^{-1} \exp\left(-\frac{1}{2}(\vu-\vv)^\top\Sigma^{-1}(\vu-\vv)\right) } \\
        &= \exp(-\frac{1}{2} \vu^\top\Sigma^{-1}\vu + \frac{1}{2}(\vu-\vv)^\top\Sigma^{-1}(\vu-\vv)) \\
        &=\exp(-\frac{1}{2} \vu^\top\Sigma^{-1}\vu + \frac{1}{2} \vu^\top\Sigma^{-1}\vu -\vu^\top\Sigma^{-1}\vv + \frac{1}{2}\vv^\top\Sigma^{-1}\vv)\\
        &= \exp(-\vu^\top\Sigma^{-1}\vv + \frac{1}{2}\vv^\top\Sigma^{-1}\vv).
    \end{align}
    Hence,
    \begin{align}
        \omega(\vu)
        &= \frac{Z^{-1} \exp(-J(\vu)) p_0(\vu)}{r(\vu)} \\
        &= Z^{-1} \exp(-[ J(\vu) + \vu^\top\Sigma^{-1}\vv]) \exp(\frac{1}{2} \vv\T \Sigma^{-1} \vv) \\
        &\propto \exp(-[ J(\vu) + \vu^\top\Sigma^{-1}\vv]),
    \end{align}
    where the last line follows since terms not dependent on $\vu$ can be canceled out between the numerator and denominator in \eqref{eq:vi:omega_tilde}.
\end{proof}

\section{NCBF Training Details}

The definition of the avoid set is given by \eqref{eqn:AvoidSet}, 
\begin{align}
    \mathcal{A} = \{x|h(x) > 0 \}, \nonumber
\end{align}
where the avoidance heuristic $h(x)$ is used to define the avoid set. 
Furthermore, the definition of the policy value function $V^{h,\pi}(x_0)$ given by \eqref{eqn:ValueFunctionDefinition}, and the ensuing training losses \eqref{eqn:DPcost}, \eqref{eqn:DPcostModified} are all dependent on $h(x)$. 
A suitable selection of the avoidance heuristic $h(x)$ must ensure that its corresponding avoidance set $\mathcal{A}$ encompasses all system states that overlap with any obstacles.
Additionally, the heuristic should improve the quality of information captured by the loss function \eqref{eqn:DPcostModified}, thereby making it more effective for neural network training. 
To this end, a reasonable choice of the avoidance heuristic utilizes coordinates in \eqref{eqn:autorally_state_space},
\begin{align}
    h_0(x) = w_I^2 - e_y^2,
\end{align}
where $w_I = 1.5~\mathrm{m}$ is half of the track width, and $e_y$ is the lateral deviation of the vehicle CoM to the track centerline. The visualization of $h_0(x)$ is demonstrated by the orange curve in \Cref{fig:AvoidanceHeuristicVisualization}.

\begin{figure}[!h]
    \centering
    \centerline{\includegraphics[scale=0.23]{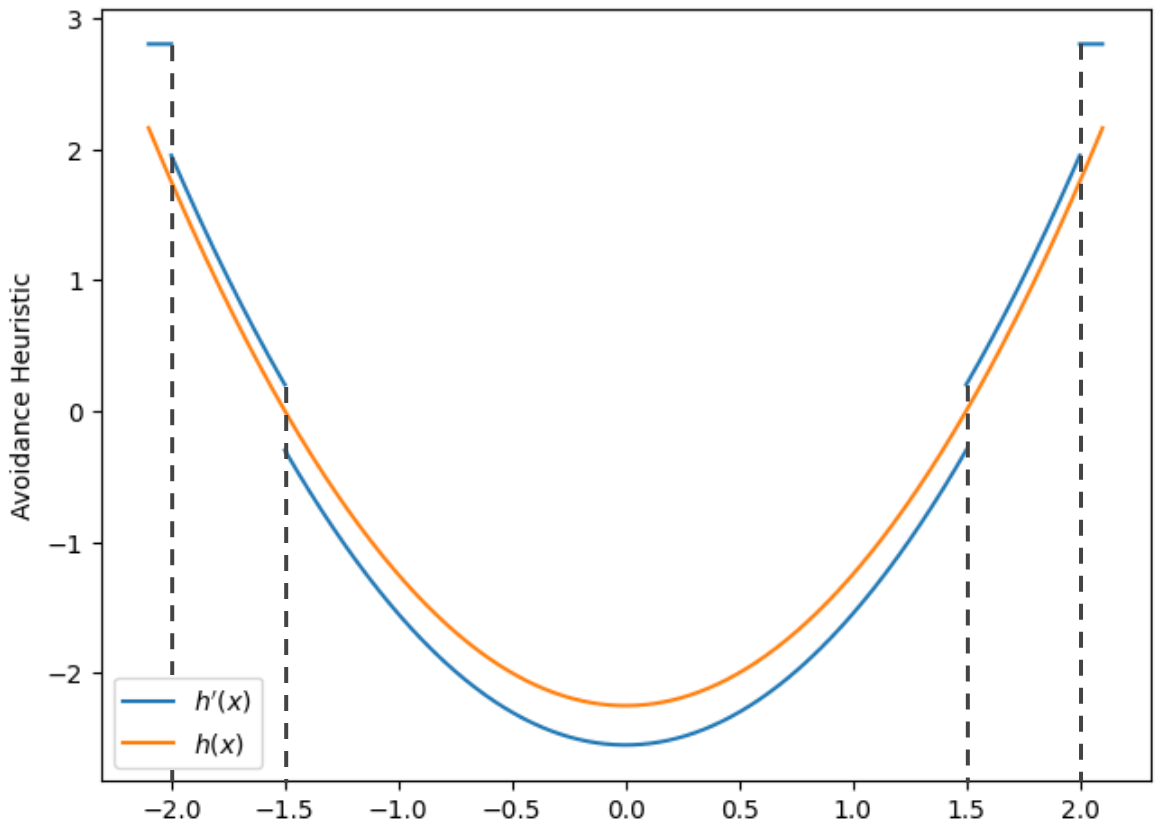}}
    \caption{Avoidance heuristic visualization. The orange curve shows the original avoidance heuristic $h_0$, while the blue curve demonstrates the modified avoidance heuristic $h$ used for training the NCBF.}
    \label{fig:AvoidanceHeuristicVisualization}
\end{figure}
However, the neural network approximation $V_\theta^{h_0, \pi}(x)$ of the policy value function $V^{h_0, \pi}$ given by \eqref{eqn:ValueFunctionDefinition} may not accurately distinguish the unsafe states in the avoid set $\mathcal{A}_0 = \{x|h_0(x) > 0 \}$ from the rest of the state space, due to the fact that trained neural networks can suffer from insufficient accuracy and precision \cite{jiang2023neural, de2018high}. This is demonstrated by the orange line and its error margins in Fig.\ref{fig:closeup_plot}. To alleviate this problem, we introduce discontinuity to $h_0(x)$ for enhanced tolerance of policy value function modelling errors while maintaining the same avoid set \eqref{eqn:AvoidSet}, such that,
\begin{align}
    \mathcal{A} = \{x|h(x) > 0 \} = \{x|h_0(x) > 0 \}.
\end{align}
where our choice of the avoidance heuristic is given by,
\begin{align}
h(x) = 
\begin{cases} 
h_0(x) - 0.3, & \text{if } e_y < w_I, \\
h_0(x) + 0.2, & \text{if } w_I \leq e_y \leq w_O,\\
2.8, & \text{if } w_O \leq e_y,
\end{cases}
\end{align}
where $w_O$ is the``crash width''. If $w_O < e_y$, the vehicle crashes and needs to be reset. If $ w_I \leq e_y \leq w_O$, the AutoRally collides with the soft drainage pipes but can still continue driving. 
As shown by the blue curve in Fig. \ref{fig:AvoidanceHeuristicVisualization}, $h(x)$ is designed to reduce the error of the avoid set of the resulting NCBF by introducing a discontinuity around zero. This is further demonstrated in Fig. \ref{fig:closeup_plot}.
\begin{figure}[!h]
    \centering
    \centerline{\includegraphics[scale=0.25]{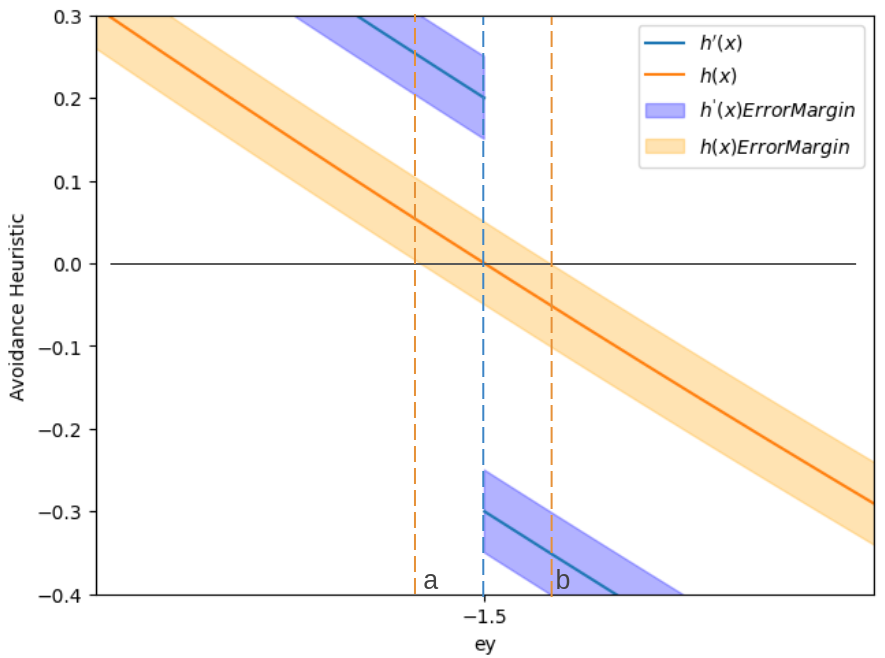}}
    \caption{Improving accuracy and precision of NCBF modeled avoid set boundary by using the modified heuristic $h(x)$. When using the original avoidance heuristic $h_0(x)$ to supervise the NCBF training process, the resulting modeled avoid set boundary can be anywhere within $e_y \in (a, b)$ due to model errors. Instead, using $h(x)$ to supervise the training process results in a modeled avoid set with an accurate boundary shown by the dashed blue line, despite model errors.}
    \label{fig:closeup_plot}
\end{figure}

\begin{figure}[!h]
    \centering
    \centerline{\includegraphics[scale=0.25]{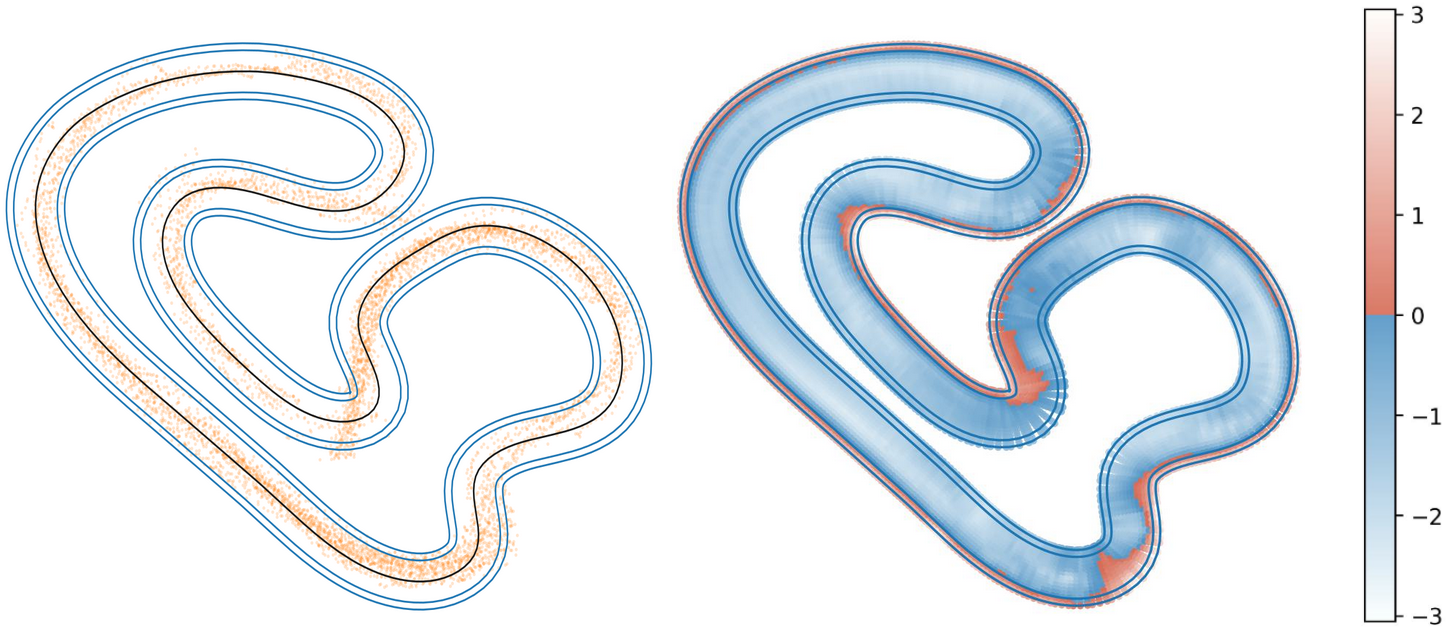}}
    \caption{NCBF Training. The left figure shows  collected system states (yellow dots) in AutoRally simulations. The right figure visualizes the resulting Neural CBF. The NCBF $B(x)$ takes the 8-dimensional state $x$ as input. 
    The red color indicates an unsafe region where $B(x) > 0$, while the blue color indicates a safe region where $B(x) \leq 0$.}
    \label{fig:NCBFTraining}
\end{figure}

\begin{figure}[!h]
    \centering
    \centerline{\includegraphics[scale = 0.34]{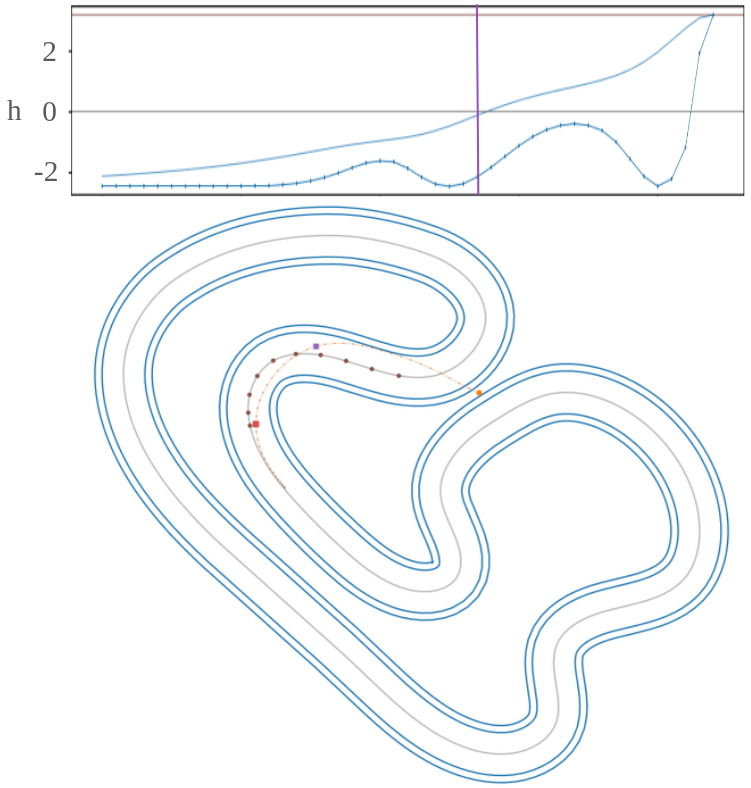}}
    \caption{A simulation example of the value change of the learned Neural CBF $B(x)$ compared to the avoid set heuristic $h(x)$. The orange trajectory produced by the standard MPPI results in a collision with the track boundary. The square purple dot shows the point where $B(x)$ shifts from negative to positive values, corresponding to the purple vertical line in the curve plot above, which shows the value change of the heuristics $B(x)$, $h(x)$ along the MPPI trajectory. The dots along the track centerline are observations collected starting at the square red dot, used to augment the training data.}
    \label{fig:NCBFexample}
\end{figure}
 The orange dots in the left plot in Fig. \ref{fig:NCBFTraining} represent the vehicle states collected using the training data. The dots outside of the track show that the autonomous car went off the track at several turns. These data are used to train a neural DCBF, leading to a visualization as demonstrated by the plot on the right in Fig. \ref{fig:NCBFTraining},in which the NCBF $B(x)$ is mapped onto the 2D Cartesian plane by utilizing mean values from the training dataset for the remaining state dimensions. The red color shows the area where $B(x) > 0$ , and as a result, the neural CBF slows down the Autorally when it approaches sharp turns to ensure safety.

From Fig. \ref{fig:NCBFexample}, we see that the learned NCBF $B(x)$ has a smaller safe set (negative level set) than the avoid set heuristic $h(x)$, as it turns positive earlier than $h(x)$ when the vehicle approaches the track boundary. When used as a safety filter, $B(x)$ can slow down the autonomous vehicle much earlier, making its negative level set control-invariant. 

\section{Proofs For Resampling-based Rollouts (RBR) supervised by a CBF}
\subsection{Proof of Theorem~\ref{thm:forward_invariant}}
\begin{proof} \label{proof3.1}
Condition \eqref{eqn:descent_condition} implies that $B(x_{k}) \leq (\text{Id} - \alpha) \circ B(x_{k-1}) $, where $\circ$ denotes function composition and $\mathrm{Id}$
denotes the identity function~\cite{DCBF}. 
Since $B(x_{1}) \leq (\text{Id} - \alpha) \circ B(x_0) $, it follows that,
\begin{equation}\label{induction}
    B(x_k) \leq (\text{Id}-\alpha)^k \circ B(x_0).
\end{equation}
Since $(\text{Id} - \alpha)$ is a class-$\kappa$ function for $a \in (0, 1)$, it follows from $B(x_0) \leq 0$ that $B(x_k) \leq 0$. Hence, the set $\mathcal{S}$ is forward invariant for system~\eqref{dynamics}.
\end{proof}

\subsection{Proof of Theorem~\ref{thm:same_marginal}} \label{sec:proof:same_marginal}
\begin{proof}
    Note that
    \begin{equation}
        \mathsf{f}(\tilde{\mathsf{b}} \mid \mathsf{a}, \mathsf{b})
        = \mathbbm{1}_{\mathsf{b} \in \mathsf{S}} \delta(\tilde{\mathsf{b}} - \mathsf{b})
        + \mathbbm{1}_{\mathsf{b} \not\in \mathsf{S}} \delta(\tilde{\mathsf{b}} - \mathsf{a}).
    \end{equation}
    Let $\tilde{\mathsf{f}}$ denote the marginal density of $\tilde{\mathsf{b}}$.
    Then, using the independence of $\mathsf{a}$ and $\mathsf{b}$,
    \begin{equation}
        \tilde{\mathsf{f}}(\tilde{\mathsf{b}})
        = \int \left( \int \mathsf{f}(\tilde{\mathsf{b}} \mid \mathsf{a}, \mathsf{b}) \mathsf{f}(\mathsf{a} \mid \mathsf{a} \in \mathsf{S}) \diff{\mathsf{a}} \right) \mathsf{f}(\mathsf{b}) \diff{\mathsf{b}}
    \end{equation}
    Simplifying the inner integral first using properties of the Dirac delta function gives
    \begin{align}
        &\mathrel{\phantom{=}} \int \mathsf{f}(\tilde{\mathsf{b}} \mid \mathsf{a}, \mathsf{b}) \mathsf{f}(\mathsf{a} \mid \mathsf{a} \in \mathsf{S}) \diff{\mathsf{a}} \\
        &= \mathbbm{1}_{\mathsf{b} \in \mathsf{S}} \delta(\tilde{\mathsf{b}} - \mathsf{b}) + 
        \mathbbm{1}_{\mathsf{b} \not\in \mathsf{S}} \int 
            \delta(\tilde{\mathsf{b}} - \mathsf{a})
            \mathsf{f}(\mathsf{a} \mid \mathsf{a} \in \mathsf{S}) \diff{\mathsf{a}} \\
        &= \mathbbm{1}_{\mathsf{b} \in \mathsf{S}} \delta(\tilde{\mathsf{b}} - \mathsf{b}) + 
        \mathbbm{1}_{\mathsf{b} \not\in \mathsf{S}} \mathsf{f}(\tilde{\mathsf{b}} \mid \tilde{\mathsf{b}} \in \mathsf{S}).
    \end{align}
    Hence,
    \begin{align}
        \tilde{\mathsf{f}}(\tilde{\mathsf{b}})
        &= \int \left(
            \mathbbm{1}_{\mathsf{b} \in \mathsf{S}} \delta(\tilde{\mathsf{b}} - \mathsf{b}) + 
            \mathbbm{1}_{\mathsf{b} \not\in \mathsf{S}} \mathsf{f}(\tilde{\mathsf{b}} \mid \tilde{\mathsf{b}} \in \mathsf{S})
        \right) \mathsf{f}(\mathsf{b}) \diff{\mathsf{b}} \\
        &= \mathsf{f}(\tilde{\mathsf{b}}) \mathbbm{1}_{\tilde{\mathsf{b}} \in \mathsf{S}}
        + \mathsf{f}(\tilde{\mathsf{b}} \mid \tilde{\mathsf{b}} \in \mathsf{S})
        \int \mathbbm{1}_{\mathsf{b} \not\in \mathsf{S}} \mathsf{f}(\mathsf{b}) \diff{\mathsf{b}} \\
        &= \mathsf{f}(\tilde{\mathsf{b}} \mid \tilde{\mathsf{b}} \in \mathsf{S})
        P(\mathsf{b} \in \mathsf{S})
        + \mathsf{f}(\tilde{\mathsf{b}} \mid \tilde{\mathsf{b}} \in \mathsf{S})
        P(\mathsf{b} \not\in \mathsf{S}) \\
        &= \mathsf{f}(\tilde{\mathsf{b}} \mid \tilde{\mathsf{b}} \in \mathsf{S}).
    \end{align}
\end{proof}

\subsection{Proof of Corollary~\ref{thm:resample_unbiased}} \label{sec:proof:resample_unbiased}
\begin{proof}
    Applying \Cref{thm:same_marginal}, $\mathsf{x}, \mathsf{a}$ and $\tilde{\mathsf{b}}$ have the same distribution,
    \begin{equation}
        \mathsf{x} \stackrel{d}{=} \mathsf{a} \stackrel{d}{=} \tilde{\mathsf{b}}.
    \end{equation}
    Hence,
    \begin{align}
        \E[\frac{1}{2} w( \mathsf{a} ) + \frac{1}{2} w( \tilde{\mathsf{b}} )]
        &= \frac{1}{2} \E[ w( \mathsf{a} )] + \frac{1}{2} \E[ w( \tilde{\mathsf{b}} )], \\
        &= \frac{1}{2} \E[ w( \mathsf{x} )] + \frac{1}{2} \E[ w( \mathsf{x} )], \\
        &= \E[w( \mathsf{x} )].
    \end{align}
\end{proof}

\subsection{Proof of Theorem~\ref{thm:resample_variance}} \label{sec:proof:resample_variance}
We prove the two claims separately.

\begin{lemma}
    The variance of the Monte Carlo estimator of the optimal control law is
    \begin{equation}
        \Var[\hat{v}_k] = \frac{1}{N}\left( \frac{1}{3} 2^K - \frac{1}{4} \right), \qquad k = 1, \dots, K.
    \end{equation}
\end{lemma}
\begin{proof}
First, note that the optimal control $\vv^*$ is given by
\begin{align}
    v^*_k
    &= \int_{[-1, 1]^K} p(\vu \mid o=1) u_k \diff{\vu} \\
    &= \int_{[-1, 1]^K} \ind{\vu \in [0, 1]} u_k \diff{\vu} \\
    &= \int_{[0, 1]^K} u_k \diff{\vu} \\
    &= \frac{1}{2}.
\end{align}
Using the formula for computing the variance of an importance sampled Monte Carlo estimator \cite{owen2013monte}, we then have that
\begin{align}
    \Var[\hat{v}_k]
    &= \frac{1}{N} \int_{[-1, 1]^K} \frac{ ( u_k p(\vu \mid o=1) - v^*_k r(\vu) )^2 }{ r(\vu) } \diff{\vu} \\
    &= \frac{1}{N} \int_{[-1, 1]^K} \frac{ ( u_k p(\vu \mid o=1) - \frac{1}{2} 2^{-K} )^2 }{ 2^{-K} } \diff{\vu} \\
    &= \frac{1}{N} 2^K \left( \int_{[0, 1]^K} (u_k)^2 \diff{\vu} - 2^{-K} \int_{[0, 1]^K} u_k \diff{\vu} + \frac{1}{4} (2^{-K})^2 \int_{[-1, 1]^K} 1 \diff{\vu} \right) \\
    &= \frac{1}{N} 2^K \left( \frac{1}{3} - \frac{1}{2} 2^{-K} + \frac{1}{4} 2^{-K} \right) \\
    &= \frac{1}{N}\left( \frac{1}{3} 2^K - \frac{1}{4} \right).
\end{align}
\end{proof}

\noindent\textbf{Proof of second claim.}
We now prove the second claim.
When performing resampling-based rollouts, \textit{only} if $u^i_k$ lies in $[-1, 0]$ for all $k = 0, \dots, K - 2$ does resampling not occur, and the output $\tilde{u}^i_k \in [-1, 0]$ for $k = 0, \dots, K-2$.
Otherwise, we have that $\tilde{u}^i_k \in [0, 1]$ for $k = 0, \dots, K-2$.
The last control $u^i_{K-1} = \tilde{u}^i_{K-1}$ is never resampled.
Hence, the probability that all $N$ samples lie in $[-1, 0]^{K-1} \times [-1, 1]$ is equal to $2^{-N(K-1)}$.

For all timesteps except the last, we write the joint probability density function (over all $N$ samples) as
\begin{equation}
    p(\tilde{u}^1_k, \dots, \tilde{u}^N_k) = \begin{dcases}
        2^{-N}, & \text{if } \tilde{u}^i_k \in [-1, 0] \text{ for all } i = 1, \dots, N, \\
        1 - 2^{-N}, & \text{if } \tilde{u}^i_k \in [0, 1] \text{ for all } i = 1, \dots, N, \\
        0, & \text{otherwise,}
    \end{dcases}
    \qquad
    \quad
    \text{ for }
    k = 0, \dots, K - 2,
\end{equation}
and
\begin{equation}
    p(\tilde{u}^i_{K-1}) = 1,
    \qquad
    p(\tilde{u}^1_{K-1}, \dots, \tilde{u}^N_{K-1}) = \prod_{i=1}^N p(\tilde{u}^i_{K-1}) = 1.
\end{equation}
For convenience, let $t \coloneqq 2^{-N}$ such that $P(\tilde{u}^{1:N}_k \in [0, 1]) = 1 - t$ for $k < K-1$,
and
\begin{equation}
    P(\tilde{u}^{1:N}_{0:K-2} \in [0, 1]) = \prod_{k=0}^{K-2} P(\tilde{u}^{1:N}_k \in [0, 1]) = (1 - t)^{K-1}.
\end{equation}
Next, consider the Monte Carlo estimator $\hat{v}_k$ using the resampled controls $\tilde{u}^i_k$, defined by
\begin{equation}
    \hat{v}_k = \frac{1}{N} \sum_{i=1}^N \frac{ \ind{\tilde{u}^i_{0:K-1} \in [0, 1]} }{ \frac{1}{2} (1 - t)^{K-1} } \tilde{u}^i_k.
\end{equation}
We will compute the expectation of $\hat{v}_k$ using the law of total expectation. To this end, we have that
\begin{align}
    \E_{\tilde{u}_{0:K-1}^{1:N}}[ \hat{v}_k ]
    &= \frac{1}{N} \E_{\tilde{u}_{0:K-1}^{1:N}}\left[ \sum_{i=1}^N \frac{ \ind{\tilde{u}^i_{0:K-1} \in [0, 1]} }{ \frac{1}{2} (1 - t)^{K-1} } \tilde{u}^i_k \right] \\
    &= \frac{1}{N}\Big( P(\tilde{u}_{0:K-2}^{1:N} \geq 0) \E_{\tilde{u}_{0:K-1}^{1:N}}\left[ \sum_{i=1}^N \frac{ \ind{\tilde{u}^i_{0:K-1} \in [0, 1]} }{ \frac{1}{2} (1 - t)^{K-1} } \tilde{u}^i_k \mathrel{\bigg|} \tilde{u}_{0:K-2}^{1:N} \geq 0\right] \\
    &\quad + P(\tilde{u}_{0:K-2}^{1:N} < 0) \E_{\tilde{u}_{0:K-1}^{1:N}}\left[ \sum_{i=1}^N \frac{ \ind{\tilde{u}^i_{0:K-1} \in [0, 1]} }{ \frac{1}{2} (1 - t)^{K-1} } \tilde{u}^i_k \mathrel{\bigg|} \tilde{u}_{0:K-2}^{1:N} < 0\right]
    \Big) \notag \\
    &= \frac{1}{N}\left( (1 - t)^{K-1} \E_{\tilde{u}_{0:K-1}^{1:N}}\left[ \sum_{i=1}^N \frac{ \ind{\tilde{u}^i_{0:K-1} \in [0, 1]} }{ \frac{1}{2} (1 - t)^{K-1} } \tilde{u}^i_k \mathrel{\bigg|} \tilde{u}_{0:K-2}^{1:N} \geq 0\right]
    \right) \\
    &= \frac{2}{N}\left( \E_{\tilde{u}_{0:K-1}^{1:N}}\left[ \sum_{i=1}^N \ind{\tilde{u}^i_{0:K-1} \in [0, 1]}\, \tilde{u}^i_k \mathrel{\bigg|} \tilde{u}_{0:K-2}^{1:N} \geq 0\right]
    \right).
\end{align}
We now split into two cases. When $k = K-1$,
\begin{align}
    \E_{\tilde{u}_{0:K-1}^{1:N}}[ \hat{v}_{K-1} ]
    &= \frac{2}{N}\left( \E_{\tilde{u}_{K-1}^{1:N}}\left[ \E_{\tilde{u}_{0:K-2}^{1:N}}\left[ \sum_{i=1}^N \ind{\tilde{u}^i_{0:K-1} \in [0, 1]}\, \tilde{u}^i_{K-1} \mathrel{\bigg|} \tilde{u}_{0:K-2}^{1:N} \geq 0\right] \right]
    \right) \\
    &= \frac{2}{N}\left( \E_{\tilde{u}_{K-1}^{1:N}}\left[ \sum_{i=1}^N \ind{\tilde{u}^i_{K-1} \in [0, 1]}\, \tilde{u}^i_{K-1} \right]
    \right) \\
    &= \frac{2}{N}\left( \sum_{i=1}^N \E_{\tilde{u}_{K-1}^{i}}\left[ \ind{\tilde{u}^i_{K-1} \in [0, 1]} \tilde{u}^i_{K-1} \right]
    \right) \\
    &= \frac{2}{N}\left( \sum_{i=1}^N \frac{1}{4} \right) \\
    &= \frac{1}{2}.
\end{align}
Similarly, when $k < K-1$,
\begin{align}
    \E_{\tilde{u}_{0:K-1}^{1:N}}[ \hat{v}_{k} ]
    &= \frac{2}{N}\left( \E_{\tilde{u}_{K-1}^{1:N}}\left[ \E_{\tilde{u}_{0:K-2}^{1:N}}\left[ \sum_{i=1}^N \ind{\tilde{u}^i_{K-1} \in [0, 1]} \ind{\tilde{u}^i_{0:K-2} \in [0, 1]}\, \tilde{u}^i_{k} \mathrel{\bigg|} \tilde{u}_{0:K-2}^{1:N} \geq 0\right] \right]
    \right) \\
    &= \frac{1}{N}\left( \E_{\tilde{u}_{0:K-2}^{1:N}}\left[ \sum_{i=1}^N \tilde{u}^i_{k} \mathrel{\bigg|} \tilde{u}_{0:K-2}^{1:N} \geq 0\right]
    \right) \\
    &= \frac{1}{2}.
\end{align}
Thus, $\hat{v}_k$ is an unbiased estimator.

Next, we look at the variance.
Summarizing the above computation, we have that
\begin{equation}
    \E_{\tilde{u}_{0:K-1}^{1:N}}[ \hat{v}_k \mid \tilde{u}_{0:K-2}^{1:N} \geq 0 ]
    = \frac{1}{2 (1-t)^{K-1} }, \qquad 
    \E_{\tilde{u}_{0:K-1}^{1:N}}[ \hat{v}_k \mid \tilde{u}_{0:K-2}^{1:N} < 0 ] = 0.
\end{equation}
Let $c \coloneqq \E_{\tilde{u}_{0:K-1}^{1:N}}[ \hat{v}_k \mid \tilde{u}_{0:K-2}^{1:N} \geq 0 ]
    = \frac{1}{2 (1-t)^{K-1} }$ for convenience.
Computing the conditional variances, for $\tilde{u}_{0:K-2}^{1:N} < 0$, we have that
\begin{equation}
    \Var_{\tilde{u}_{0:K-1}^{1:N}}[ \hat{v}_k \mid \tilde{u}_{0:K-2}^{1:N} < 0 ] = 0.
\end{equation}
For $\tilde{u}_{0:K-2}^{1:N} \geq 0 $ we have,
\begin{align}
    \Var_{\tilde{u}_{0:K-1}^{1:N}}[ \hat{v}_k \mid \tilde{u}_{0:K-2}^{1:N} \geq 0 ]
    &= \E_{\tilde{u}_{0:K-1}^{1:N}}[ \hat{v}_k^2 \mid \tilde{u}_{0:K-2}^{1:N} \geq 0 ] - c^2 \\
    &= \E_{\tilde{u}_{0:K-1}^{1:N}}\left[
        \left( \frac{1}{N} \sum_{i=1}^N \frac{ \ind{\tilde{u}^i_{0:K-1} \in [0, 1]} }{ \frac{1}{2} (1 - t)^{K-1} } \tilde{u}^i_k \right)^2
    \mid \tilde{u}_{0:K-2}^{1:N} \geq 0 \right] - c^2 \\
    &= \left( \frac{1}{N} \frac{1}{ \frac{1}{2} (1-t)^{K-1} } \right)^2
        \underbrace{\E_{\tilde{u}_{0:K-1}^{1:N}}\left[
            \left( \sum_{i=1}^N \ind{\tilde{u}^i_{0:K-1} \in [0, 1]} \tilde{u}^i_k \right)^2
    \mathrel{\bigg|} \tilde{u}_{0:K-2}^{1:N} \geq 0 \right]}_{\coloneqq\; \circled{$\star$}} - c^2\label{eq:star}
\end{align}
Expanding $\circled{$\star$}$, yields
\begin{equation} \label{eq: var_proof:star}
    \circled{$\star$}
    = \E_{\tilde{u}_{0:K-1}^{1:N}}\left[
            \sum_{i=1}^N \ind{\tilde{u}^i_{0:K-1} \in [0, 1]} (\tilde{u}^i_k)^2 +
            \sum_{i=1}^N \sum_{j=1, j\not=i}^N \ind{\tilde{u}^i_{0:K-1} \in [0, 1]} \ind{\tilde{u}^j_{0:K-1} \in [0, 1]} \tilde{u}^i_k \tilde{u}^j_k
    \mathrel{\bigg|} \tilde{u}_{0:K-2}^{1:N} \geq 0 \right].
\end{equation}
We now split into two cases.

\vspace{.5\baselineskip}

\noindent\textbf{Case 1: $k = K-1$}.
Looking at the first term of \circled{$\star$} in \eqref{eq: var_proof:star},
\begin{align}
    \E_{\tilde{u}_{0:K-1}^{1:N}}\left[
        \sum_{i=1}^N \ind{\tilde{u}^i_{0:K-1} \in [0, 1]} (\tilde{u}^i_{K-1})^2
        \mid \tilde{u}_{0:K-2}^{1:N} \geq 0
    \right]
    &= \E_{\tilde{u}_{0:K-1}^{1:N}}\left[
        \sum_{i=1}^N \E_{\tilde{u}^i_{K-1}}[ \ind{\tilde{u}^i_{K-1} \in [0, 1]}  (\tilde{u}^i_{K-1})^2 ]
        \mid \tilde{u}_{0:K-2}^{1:N} \geq 0
    \right] \\
    &= \sum_{i=1}^N \E_{\tilde{u}^i_{K-1}}[ \ind{\tilde{u}^i_{K-1} \in [0, 1]}  (\tilde{u}^i_{K-1})^2 ] \\
    &= \frac{N}{6}.\label{eq:term1}
\end{align}
For the second term in \eqref{eq: var_proof:star}, we have
\begin{align}
    &\mathrel{\phantom{=}} \E_{\tilde{u}_{0:K-1}^{1:N}}\left[
        \sum_{i=1}^N \sum_{j=1, j\not=i}^N \ind{\tilde{u}^i_{0:K-1} \in [0, 1]} \ind{\tilde{u}^j_{0:K-1} \in [0, 1]} \tilde{u}^i_{K-1} \tilde{u}^j_{K-1}
        \mid \tilde{u}_{0:K-2}^{1:N} \geq 0
    \right] \\
    &= \sum_{i=1}^N \sum_{j=1, j\not=i}^N \E_{\tilde{u}_{K-1}^{i,j}}\left[ \ind{\tilde{u}^i_{K-1} \in [0, 1]} \ind{\tilde{u}^j_{K-1} \in [0, 1]} \tilde{u}^i_{K-1} \tilde{u}^j_{K-1} \right] \\
    &= \sum_{i=1}^N \sum_{j=1, j\not=i}^N \E_{\tilde{u}_{K-1}^{i,j}}\left[ \ind{\tilde{u}^i_{K-1} \in [0, 1]} \tilde{u}^i_{K-1} \right]^2 \\
    &= \sum_{i=1}^N \sum_{j=1, j\not=i}^N \frac{1}{4} \\
    &= \frac{N(N-1)}{4}.\label{eq:term2}
\end{align}
Substituting the two terms \eqref{eq:term1} and \eqref{eq:term2} into \eqref{eq:star}, yields
\begin{align}
    \Var_{\tilde{u}_{0:K-1}^{1:N}}[ \hat{v}_{K-1} \mid \tilde{u}_{0:K-2}^{1:N} \geq 0 ]
    &= \left( \frac{1}{N} \frac{1}{ \frac{1}{2} (1-t)^{K-1} } \right)^2
    \left( \frac{N}{6} + \frac{N(N-1)}{4} \right) - c^2 \\
    &\leq \left( \frac{1}{N} \frac{1}{ \frac{1}{2} (1-t)^{K-1} } \right)^2
    \frac{N^2}{4} - c^2 \\
    &= \left( \frac{1}{ (1-t)^2} \right)^{K-1} - \frac{1}{4} \left( \frac{1}{ (1-t)^2} \right)^{K-1}, \\
    &= \frac{3}{4} \left( \frac{1}{ (1-t)^2} \right)^{K-1}.
\end{align}

\vspace{.5\baselineskip}

\noindent\textbf{Case 2: $k < K-1$}.
Looking at the first term of \circled{$\star$} in \eqref{eq: var_proof:star},
\begin{align}
    \E_{\tilde{u}_{0:K-1}^{1:N}}\left[
        \sum_{i=1}^N \ind{\tilde{u}^i_{0:K-1} \in [0, 1]} (\tilde{u}^i_k)^2
        \mid \tilde{u}_{0:K-2}^{1:N} \geq 0
    \right]
    &= \E_{\tilde{u}_{0:K-1}^{1:N}}\left[
        \sum_{i=1}^N \E_{\tilde{u}^i_{K-1}}[ \ind{\tilde{u}^i_{K-1} \in [0, 1]} ] (\tilde{u}^i_k)^2
        \mid \tilde{u}_{0:K-2}^{1:N} \geq 0
    \right] \\
    &= \frac{1}{2} \E_{\tilde{u}_{0:K-1}^{1:N}}\left[
        \sum_{i=1}^N (\tilde{u}^i_k)^2
        \mid \tilde{u}_{0:K-2}^{1:N} \geq 0
    \right].
\end{align}
From \Cref{thm:same_marginal}, the marginal of any resampled controls will all be uniform over $[0, 1]$. Hence,
\begin{align}
    \frac{1}{2} \E_{\tilde{u}_{0:K-1}^{1:N}}\left[
        \sum_{i=1}^N (\tilde{u}^i_k)^2
        \mid \tilde{u}_{0:K-2}^{1:N} \geq 0
    \right]
    &= \frac{1}{2} \sum_{i=1}^N \E_{\tilde{u}_k^i}[ (\tilde{u}_k^i)^2 \mid \tilde{u}_{k}^{i} \geq 0] = \frac{N}{6}.\label{eq:term3}
\end{align}
For the second term in \eqref{eq: var_proof:star},
\begin{align}
    &\mathrel{\phantom{=}} \E_{\tilde{u}_{0:K-1}^{1:N}}\left[
        \sum_{i=1}^N \sum_{j=1, j\not=i}^N \ind{\tilde{u}^i_{0:K-1} \in [0, 1]} \ind{\tilde{u}^j_{0:K-1} \in [0, 1]} \tilde{u}^i_k \tilde{u}^j_k
        \mid \tilde{u}_{0:K-2}^{1:N} \geq 0
    \right] \\
    &= \E_{\tilde{u}_{0:K-2}^{1:N}}\left[
        \sum_{i=1}^N \sum_{j=1, j\not=i}^N \E_{\tilde{u}_{K-1}^{1:N}}\left[ \ind{\tilde{u}^i_{K-1} \in [0, 1]} \ind{\tilde{u}^j_{K-1} \in [0, 1]} \right] \tilde{u}^i_k \tilde{u}^j_k
        \mid \tilde{u}_{0:K-2}^{1:N} \geq 0
    \right] \\
    &= \frac{1}{4} \E_{\tilde{u}_{0:K-2}^{1:N}}\left[
        \sum_{i=1}^N \sum_{j=1, j\not=i}^N \tilde{u}^i_k \tilde{u}^j_k
        \mid \tilde{u}_{0:K-2}^{1:N} \geq 0
    \right]. \label{eq: var_proof:tmp1}
\end{align}
To make this computation easier, let $\Xi_{ij}=1$ denote the event that $\tilde{u}^i_k$ and $\tilde{u}^j_k$ for $i \not= j$ were resampled from the same control, i.e., $\tilde{u}^i_k = \tilde{u}^j_k$, and let $\beta = \mathbb{P}(\Xi_{ij}=1)$.
When $\Xi_{ij}=1$, $\tilde{u}^i_k$ and $\tilde{u}^j_k$ are the same random variable, and therefore it follows that
\begin{align}
    \E_{\tilde{u}_{0:K-2}^{1:N}}\left[ \tilde{u}^i_k \tilde{u}^j_k \mid \Xi_{ij}=1 \mid \tilde{u}_{0:K-2}^{1:N} \geq 0, \Xi_{ij}=1 \right]
    &= \E_{\tilde{u}^i_k}\left[ (\tilde{u}^i_k)^2 \mid \tilde{u}^i_k \geq 0 \right] = \frac{1}{3}.
\end{align}
Otherwise, when $\Xi_{ij}=0$, $\tilde{u}^i_k$ and $\tilde{u}^j_k$ are independent. Hence,
\begin{align}
    \E_{\tilde{u}_{0:K-2}^{1:N}}\left[ \tilde{u}^i_k \tilde{u}^j_k \mid \Xi_{ij}=1 \mid \tilde{u}_{0:K-2}^{1:N} \geq 0, \Xi_{ij}=0 \right]
    &= \E_{\tilde{u}^i_k}\left[ (\tilde{u}^i_k) \mid \tilde{u}^i_k \geq 0 \right]^2 = \frac{1}{4}.
\end{align}
Hence, the expectation becomes
\begin{align}
    \eqref{eq: var_proof:tmp1}
    &= \frac{1}{4} \sum_{i=1}^N \sum_{j=1, j\not=i}^N
    \beta \frac{1}{3}
    + (1-\beta) \frac{1}{4} \\
    &= \frac{1}{4} N(N-1) \left( \beta \frac{1}{3} + (1-\beta) \frac{1}{4} \right) \\
    &\leq \frac{1}{12} N(N-1).\label{eq:term4}
\end{align}
Combining the two terms \eqref{eq:term3} and \eqref{eq:term4}, we thus have that
\begin{align}
    \Var_{\tilde{u}_{0:K-1}^{1:N}}[ \hat{v}_k \mid \tilde{u}_{0:K-2}^{1:N} \geq 0 ]
    &\leq \left( \frac{1}{N} \frac{1}{ \frac{1}{2} (1-t)^{K-1} } \right)^2
    \left( \frac{N}{6} + \frac{N(N-1)}{12} \right) - c^2\\
    &\leq \left( \frac{1}{N} \frac{1}{ \frac{1}{2} (1-t)^{K-1} } \right)^2
    \frac{N^2}{6} - c^2\\
    &= \frac{2}{3} \left( \frac{1}{ (1-t)^2} \right)^{K-1} - \frac{1}{4} \left( \frac{1}{ (1-t)^2} \right)^{K-1} \\
    &= \frac{5}{12} \left( \frac{1}{ (1-t)^2} \right)^{K-1}.
\end{align}

Hence, taking both cases into account, we have that
\begin{equation}
    \Var_{\tilde{u}_{0:K-1}^{1:N}}[ \hat{v}_k \mid \tilde{u}_{0:K-2}^{1:N} \geq 0 ] \leq \frac{3}{4} \left( \frac{1}{ (1-t)^2} \right)^{K-1}.
\end{equation}
Finally, using the law of total variance, we have that
\begin{equation}
    \Var[ \hat{v}_k ] = \underbrace{\E_{\text{sign of } \tilde{u}_{0:K-2}^{1:N}}[ \Var[ \hat{v}_k \mid \tilde{u}_{0:K-2}^{1:N} ] ]}_{\circled{1}} 
    + \underbrace{\Var_{\text{sign of } \tilde{u}_{0:K-2}^{1:N}}[ \E[ \hat{v}_k \mid \tilde{u}_{0:K-2}^{1:N} ] ]}_{\circled{2}}.
\end{equation}
The first term gives
\begin{align}
    \circled{1}
    &= \mathbb{P}(\tilde{u}^{1:N}_{0:K-2} \geq 0) \Var[ \hat{v}_k \mid \tilde{u}_{0:K-2}^{1:N} \geq 0 ] + \mathbb{P}(\tilde{u}^{1:N}_{0:K-2} < 0) \Var[ \hat{v}_k \mid \tilde{u}_{0:K-2}^{1:N} < 0 ] \\
    &\leq (1 - t)^{K-1} \frac{3}{4} \left( \frac{1}{ (1-t)^2} \right)^{K-1} + 0 \\
    &= \frac{3}{4} \left(\frac{1}{1-t} \right)^{K-1}.
\end{align}
The second term gives
\begin{align}
    \circled{2}
    &= \mathbb{P}(\tilde{u}^{1:N}_{0:K-2} \geq 0) (\E[ \hat{v}_k \mid \tilde{u}_{0:K-2}^{1:N} \geq 0 ] - \E[ \hat{v}_k ])^2 + \mathbb{P}(\tilde{u}^{1:N}_{0:K-2} < 0) (\E[ \hat{v}_k \mid \tilde{u}_{0:K-2}^{1:N} < 0 ] - \E[ \hat{v}_k ])^2 \\
    &= (1 - t)^{K-1} \left( \frac{1}{2} \frac{1}{(1-t)^{K-1}} - \frac{1}{2} \right)^2 + (1 - (1 - t)^{K-1}) \frac{1}{4} \\
    &= \frac{1}{2} (1 - t)^{K-1} - 1 + \frac{1}{2} \frac{1}{(1-t)^{K-1}} + \frac{1}{4} - \frac{1}{4} (1 - t)^{K-1} \\
    &= \frac{1}{4} (1 - t)^{K-1} - \frac{3}{4} + \frac{1}{2} \frac{1}{(1-t)^{K-1}}.
\end{align}
Hence, combining both terms gives us
\begin{align}
    \Var[ \hat{v}_k ]
    &\leq \frac{3}{4} \left(\frac{1}{1-t} \right)^{K-1} + \frac{1}{4} (1 - t)^{K-1} - \frac{3}{4} + \frac{1}{2} \left(\frac{1}{1-t}\right)^{K-1} \\
    &= \frac{5}{4} \left(\frac{1}{1-t} \right)^{K-1} + \frac{1}{4} (1 - t)^{K-1} - \frac{3}{4} \\
    &= O\left( \left(\frac{1}{1 - 2^{-N}}\right)^{K} + (1-2^{-N})^{K} \right).
\end{align}

\subsection{Proof of Theorem~\ref{thm:ess}} \label{sec:proof:ess}
\begin{proof}
    We will prove the equivalent statement that
    \begin{equation} \label{eq: ess:proof:1}
        \norm{ \tilde{\vw} }^2_2 - \norm{ \tilde{\vw}' }^2_2 \geq 0.
    \end{equation}
    To begin, we expand $\tilde{\vw}'$ to obtain
    \begin{equation}
        \norm{\tilde{\vw}'}_2^2 = \frac{ \norm{\vw}_2^2 + \norm{\vc}_2^2 }{ \norm{\vw'}_1^2}.
    \end{equation}
    Then,
    \begin{align}
        \mathrel{\phantom{=}} \norm{ \tilde{\vw} }^2_2 - \norm{ \tilde{\vw}' }^2_2 &= \frac{ \norm{\vw}_2^2 }{ \norm{\vw}_1^2} + \frac{ -\norm{\vw}_2^2 - \norm{\vc}_2^2 }{ \norm{\vw'}_1^2} \\
        &= \frac{1}{ \norm{\vw'}_1^2 } \left( \norm{\vw'}_1^2 \frac{ \norm{\vw}_2^2 }{ \norm{\vw}_1^2} -\norm{\vw}_2^2 - \norm{\vc}_2^2 \right).\label{eq:parenthesis_term}
    \end{align}
    Simplifying the first term, yields,
    \begin{align}
        \mathrel{\phantom{=}} \norm{\vw'}_1^2 \frac{ \norm{\vw}_2^2 }{ \norm{\vw}_1^2} &= \left( \norm{\vw}_1^2 + 2 \norm{\vw}_1 \norm{\vc}_1 + \norm{\vc}_1^2 \right) \frac{ \norm{\vw}_2^2 }{ \norm{\vw}_1^2} \\
        &= \norm{\vw}_2^2 + 2 \frac{ \norm{ \vw }_2^2 \norm{ \vc }_1 }{ \norm{ \vw }_1 } + \frac{ \norm{ \vw }_2^2 \norm{ \vc }_1^2 }{ \norm{ \vw }_1^2 }.\label{eq:D91}
    \end{align}
    Substituting \eqref{eq:D91} back into the parenthesis \eqref{eq:parenthesis_term} gives us that
    \begin{align}
        \mathrel{\phantom{=}} \norm{\vw'}_1^2 \frac{ \norm{\vw}_2^2 }{ \norm{\vw}_1^2} -\norm{\vw}_2^2 - \norm{\vc}_2^2 &= 2 \frac{ \norm{ \vw }_2^2 \norm{ \vc }_1 }{ \norm{ \vw }_1 } + \frac{ \norm{ \vw }_2^2 \norm{ \vc }_1^2 }{ \norm{ \vw }_1^2 } - \norm{\vc}_2^2 \\
        &\geq 2 \frac{ \norm{ \vw }_2^2 \norm{ \vc }_1 }{ \norm{ \vw }_1 } + \frac{ \norm{ \vw }_2^2 \norm{ \vc }_1^2 }{ \norm{ \vw }_1^2 } - \norm{\vc}_1^2 \\
        &= \norm{\vc}_1 \left( 2 \frac{ \norm{ \vw }_2^2 }{ \norm{ \vw }_1 } + \frac{ \norm{ \vw }_2^2 \norm{ \vc }_1 }{ \norm{ \vw }_1^2 } - \norm{\vc}_1 \right).\label{eq:D94}
    \end{align}
    Simplifying the parenthesis in \eqref{eq:D94}, we obtain
    \begin{align}
        \mathrel{\phantom{=}} 2 \frac{ \norm{ \vw }_2^2 }{ \norm{ \vw }_1 } + \frac{ \norm{ \vw }_2^2 \norm{ \vc }_1 }{ \norm{ \vw }_1^2 } - \norm{\vc}_1 &= 2 \frac{ \norm{ \vw }_2^2 }{ \norm{ \vw }_1 } + \norm{\vc}_1 \left( \frac{ \norm{ \vw }_2^2 }{ \norm{ \vw }_1^2 } - 1 \right) \\
        &\geq 2 \frac{ \norm{ \vw }_2^2 }{ \norm{ \vw }_1 } + \norm{\vc}_1 \left( \frac{1}{N} - 1 \right).
    \end{align}
    Using our assumption that $\vc$ is not ``drastically larger'' than $\vw$ in \eqref{eq: ess:assumption} then gives us the desired result.
\end{proof}

\section{Details on VIMPC cost designs} \label{app:hw_details}
In both simulations and experiments, MPPI variants using DCBF to ensure safety utilized a state-dependent cost function:
\begin{equation}
    q(x_k^m) = (x_k^m - x_{g})^\intercal Q (x_k^m - x_{g}),
\end{equation}
whereas the standard MPPI employed the cost,
\begin{equation}
    q(x_k^m) = (x_k^m - x_{g})^\intercal Q (x_k^m - x_{g}) + \mathbf{1}(x_k^m),
\end{equation}
where \( Q = \diag(q_{v_x}, q_{v_y}, q_{\dot{\psi}}, q_{\omega_F}, q_{\omega_R}, q_{e_\psi}, q_{e_y}, q_s) \) represents the cost weights, and \( x_g = \text{diag}(v_g, 0, \dots, 0) \) specifies the target velocity. The collision cost function is defined as:
\begin{equation}\label{JC}
    \mathbf{1}(x_k^m) :=
    \begin{cases}
      0, & \text{if } x_k^m \text{ is within the track}, \\
      C_\text{obs}, & \text{otherwise}.
    \end{cases}
\end{equation}
Unlike standard MPPI, Shield-MPPI and NS-MPPI incorporate a DCBF constraint violation penalty but do not include an explicit collision cost. 

\newpage
\section{Constraint Satisfaction on the Variational Distribution} \label{app:convex_controls}
One way of guaranteeing that the variational distribution $q(\vu)$ satisfies the constraints $x_k \in \bar{\mathcal{X}} := \mathcal{X} \backslash \mathcal{A}$ is by making the strong (but unrealistic) assumption that the set of controls that satisfy the constraints is a convex set.
Specifically, for generic state constraints $x \in \bar{\mathcal{X}}$, define the set of safe control trajectories $\mathcal{U}^K_{\text{safe}} \subseteq \mathcal{U}^K$ as
\begin{equation}
    \mathcal{U}^K_{\text{safe}}(x_0) = \left\{ \vu \in \mathcal{U}^K \mid x_k \in \bar{\mathcal{X}},~ k =0, 1, \ldots, K \right\}.
\end{equation}
We then have the following lemma.

\begin{lemma}
    Assume $\mathcal{U}^K_{\mathrm{safe}}(x_0)$ is convex for all states $x_0 \in \bar{\mathcal{X}}$,
    and suppose that $p(\vu \mid o=1)$ has zero density outside of $\mathcal{U}^K_{\mathrm{safe}}(x_0)$.
    Then, the mean $\vv^*$ \eqref{eq:vi:omega_def} of the variational distribution $q_{\vv}$ (and the state trajectory resulting from following $\vv^*$) will also satisfy the constraints, i.e.,
    \begin{equation}
        \vv^* \in \mathcal{U}^K_{\mathrm{safe}}(x_0).
    \end{equation}
\end{lemma}

\begin{proof}
    By definition \eqref{eq:vi:v_opt_def}, $\vv^*$ is the mean of the optimal distribution $p(\vu \mid o=1)$.
    Since $p(\vu \mid o=1)$ has zero density outside $\mathcal{U}^K_{\text{safe}}(x_0)$, its support is contained within $\mathcal{U}^K_{\text{safe}}(x_0)$.
    Since $\mathcal{U}^K_{\text{safe}}(x_0)$ is convex, the integral of $\vu$ over $p(\vu \mid o=1)$ will also be contained within $\mathcal{U}^K_{\text{safe}}(x_0)$. Hence, $\vv^* \in \mathcal{U}^K_{\text{safe}}(x_0)$.
\end{proof}

\newpage
\section{RBR Performance using a Valid DCBF}\label{app:ValidDCBF}
\noindent\textbf{Valid DCBF on \textsf{\small DubinsCar}.} We evaluate a valid DCBF on \textsf{\small DubinsCar} with $x=[p_x,p_y,\theta]$, $u=[\dot{\theta}]$ (\Cref{fig:dubins_valid}).
We also add a high-cost region and a constraint that $\abs{\theta} \leq \pi/2$.
\begin{itemize}[leftmargin=0.2 em]
    \item At $K=1$, RBR has no effect (no resampling).
    With $\mathcal{U}=\mathbb{R}$,
    a valid DCBF allows safe sampling with just $N=50$, and both S-MPPI and NS-MPPI avoid crashes.
    \item As estimator variance grows exponentially with $K$, S-MPPI crashes at larger $K$.
    RBR (NS-MPPI) mitigates this, maintaining safety across all horizons.
    \item Higher $K$ is needed to avoid the high cost region. The cost slowly increases with $K\geq10$ from estimator variance. 
\end{itemize}

\begin{figure}[h]
    \centering
    \includegraphics[width=.85\linewidth]{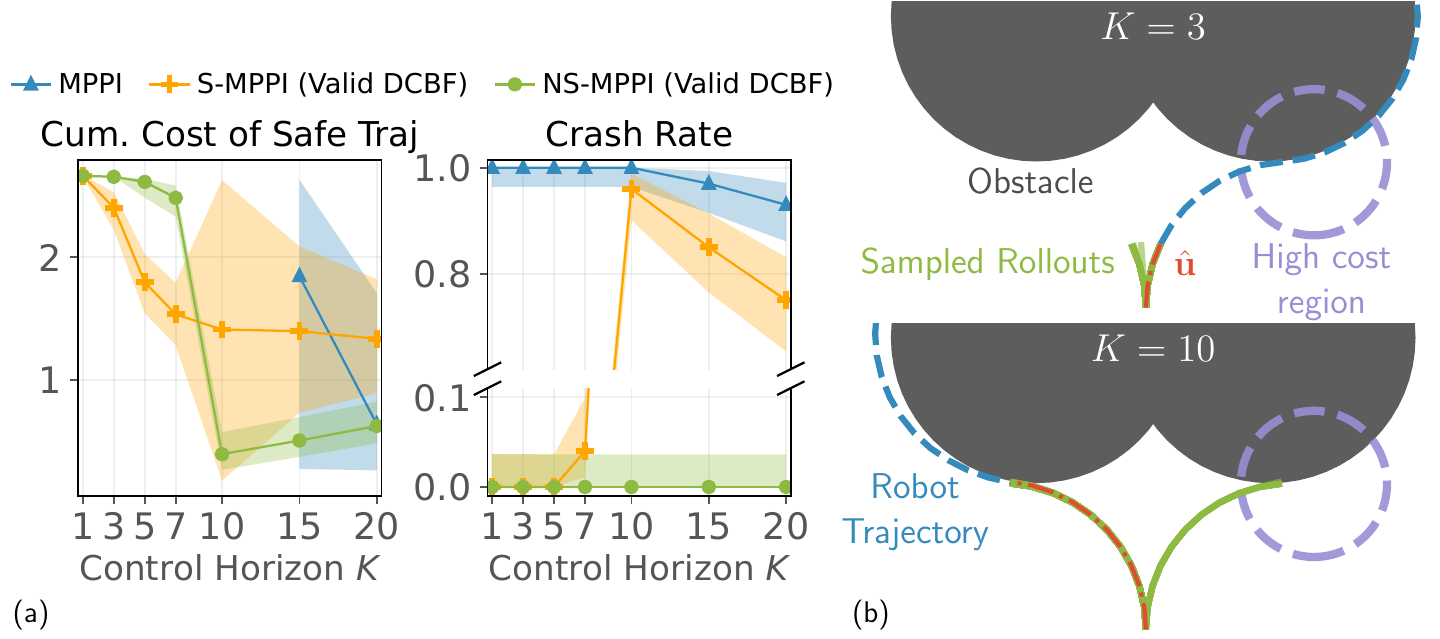}
    \caption{(a) Cumulative cost of safe trajectories using valid DCBF on \textsf{DubinsCar} with $N=50$.
    MPPI is unsafe for $K<15$ and omitted.
    (b) Higher $K$ is needed to avoid the high-cost region.}
    \label{fig:dubins_valid}
\end{figure}

\setlength{\parindent}{0pt}
\setlength{\parskip}{2pt plus1pt minus0.5pt}

\end{appendices}

\end{document}